\pdfoutput=1

\documentclass{article}

\newcommand{\cutsectionup}{\vspace*{-0.1in}}
\newcommand{\cutsectiondown}{\vspace*{-0.08in}}

\newcommand{\cutsubsectionup}{\vspace*{-0.07in}} 

\makeatletter
\newcommand{\thickhline}{%
    \noalign {\ifnum 0=`}\fi \hrule height 1.0pt
    \futurelet \reserved@a \@xhline
}

\newcommand{\abs}{{\rm AVR}}
\newcommand{\R}{\mathbb{R}}
\newcommand{\eR}{\hat{R}_L}
\newcommand{\f}{\phi}
\newcommand{\SO}{\mathcal{S}}
\newcommand{\din}{d_\textrm{in}}

\newcommand{\Expect}{\mathbf{E}}
\newcommand{\F}{\mathcal{F}}

\newcommand{\relu}{{\rm ReLU}}
\newcommand{\cat}{{\rm CReLU}}
\newcommand{\pool}{{\rm pool}}
\newcommand{\vct}[1]{\mathbf{#1}}

\newcommand{\Tm}{\tilde{\mathcal{T}}}
\newcommand{\Sm}{\tilde{\mathcal{S}}}
\newcommand{\Ttm}{\tilde{\mathcal{T}}^*}
\newcommand{\nmu}{\bar{\mu}}
\newcommand{\argmin}{\arg\!\min}
\newcommand{\cnn}{f_{\mathrm{cnn}}}

\newcommand{\T}{\mathcal{T}}

\usepackage{times}
\usepackage{graphicx} 
\usepackage{subfigure} 
\usepackage{caption}
\usepackage{enumerate}
\usepackage{lipsum,multicol}
\usepackage{multirow}
\usepackage[bottom]{footmisc}
\usepackage{enumitem}
\usepackage{hhline}
\usepackage{xfrac}
\usepackage{array}

\usepackage{natbib}
\usepackage{algorithm}
\usepackage{algorithmic}
\usepackage{amsthm}
\usepackage{amsfonts}
\usepackage{amsmath}
\usepackage{amssymb}
\usepackage{footnote}
\makesavenoteenv{tabular}
\makesavenoteenv{table}

\newtheorem{theorem}{Theorem}[section]
\newtheorem{lemma}[theorem]{Lemma}
\newtheorem{proposition}[theorem]{Proposition}
\newtheorem{corollary}[theorem]{Corollary}
\newtheorem{Definition}{Definition}[section]

\usepackage{hyperref}

\usepackage[accepted]{icml2016} 

\icmltitlerunning{Understanding and Improving Convolutional Neural Networks via Concatenated Rectified Linear Units}

\begin{document} 

\twocolumn[
\icmltitle{Understanding and Improving Convolutional Neural Networks via Concatenated Rectified Linear Units}

\icmlauthor{Wenling Shang$^{1}$}{wendy.shang@oculus.com}
\icmlauthor{Kihyuk Sohn$^{2}$}{ksohn@nec-labs.com}
\icmlauthor{Diogo Almeida$^{3}$}{diogo@enlitic.com}
\icmlauthor{Honglak Lee$^{4}$}{honglak@eecs.umich.edu}
\icmladdress{$^{1}$Oculus VR; $^{2}$NEC Laboratories America; $^{3}$Enlitic; $^{4}$University of Michigan, Ann Arbor}

\icmlkeywords{convolutional neural networks}

\vskip 0.3in
]

\begin{abstract} 
Recently, convolutional neural networks (CNNs) have been used as a powerful tool to solve many problems of machine learning and computer vision.
In this paper, we aim to provide insight on the property of convolutional neural networks, as well as a generic method to improve the performance of many CNN architectures.
Specifically, we first examine existing CNN models and observe an intriguing property that the filters in the lower layers form pairs (i.e., filters with opposite phase).
Inspired by our observation, we propose a novel, simple yet effective activation scheme called \emph{concatenated ReLU} ($\cat$) and theoretically analyze its reconstruction property in CNNs.
We integrate $\cat$ into several state-of-the-art CNN architectures and demonstrate improvement in their recognition performance on CIFAR-10/100 and ImageNet datasets with fewer trainable parameters.
Our results suggest that better understanding of the properties of CNNs can lead to significant performance improvement with a simple modification. 
%
%
%
%
%
\end{abstract} 

\section{Introduction}
\label{sec:intro}
In recent years, convolutional neural networks (CNNs) have achieved great success in many problems of machine learning and computer vision~\cite{krizhevsky2012imagenet,simonyan2014very,szegedy2015going,girshick2014rich}.
%
In addition, a wide range of techniques have been developed to enhance the performance or ease the training of CNNs~\cite{lin2013network,zeiler2013stochastic, maas2013rectifier, ioffe2015batch}.
%
%
Despite the great empirical success, fundamental understanding of CNNs is still lagging behind. 
%
Towards addressing this issue, this paper aims to provide insight on the intrinsic property of convolutional neural networks. 

%
%
%
\begin{figure}
\centering
\includegraphics[width=0.475\textwidth]{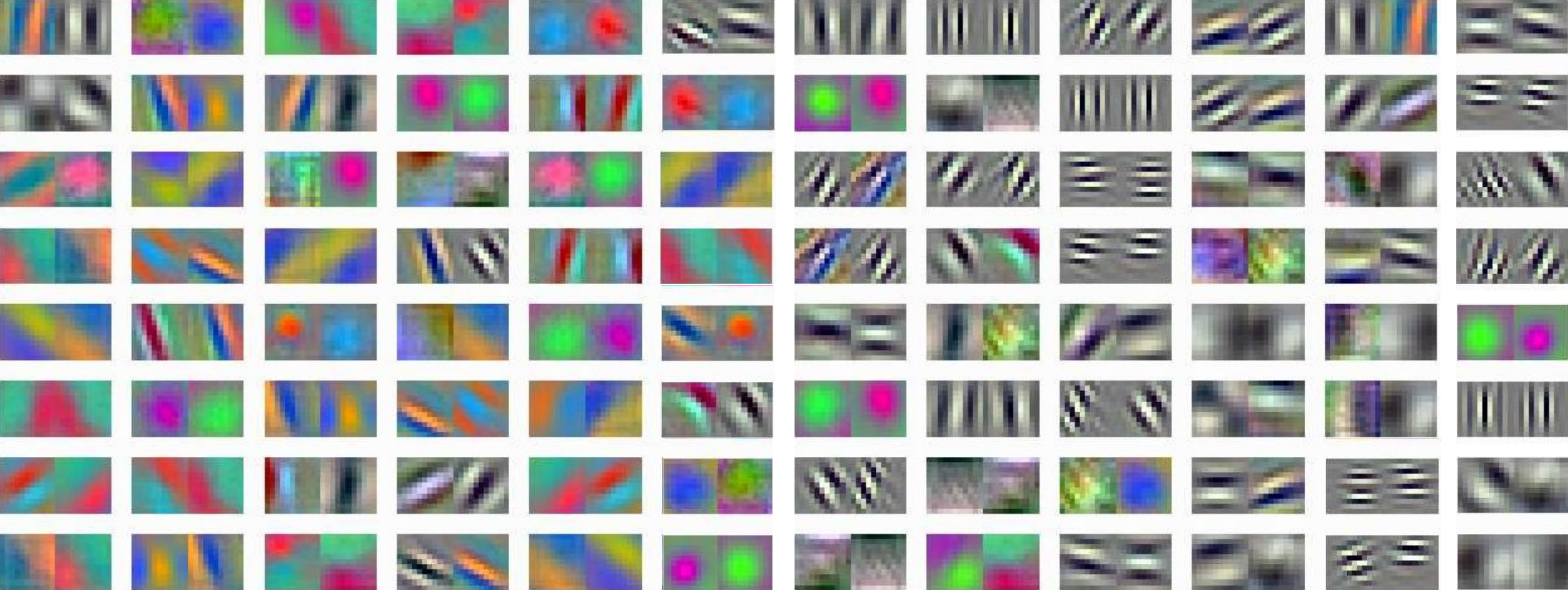}
\caption{\textbf{Visualization of conv1 filters from AlexNet}. Each filter and its pairing filter ($w_i$ and $\bar{w}_i$ next to each other) appear surprisingly opposite (in phase) to each other. See text for details.}\label{fig:CNN_pairs_Alex}
\vspace*{-0.15in}
\end{figure}

\begin{figure*}[htbp]
\centering
\subfigure[conv1]{\includegraphics[width=0.19\textwidth]{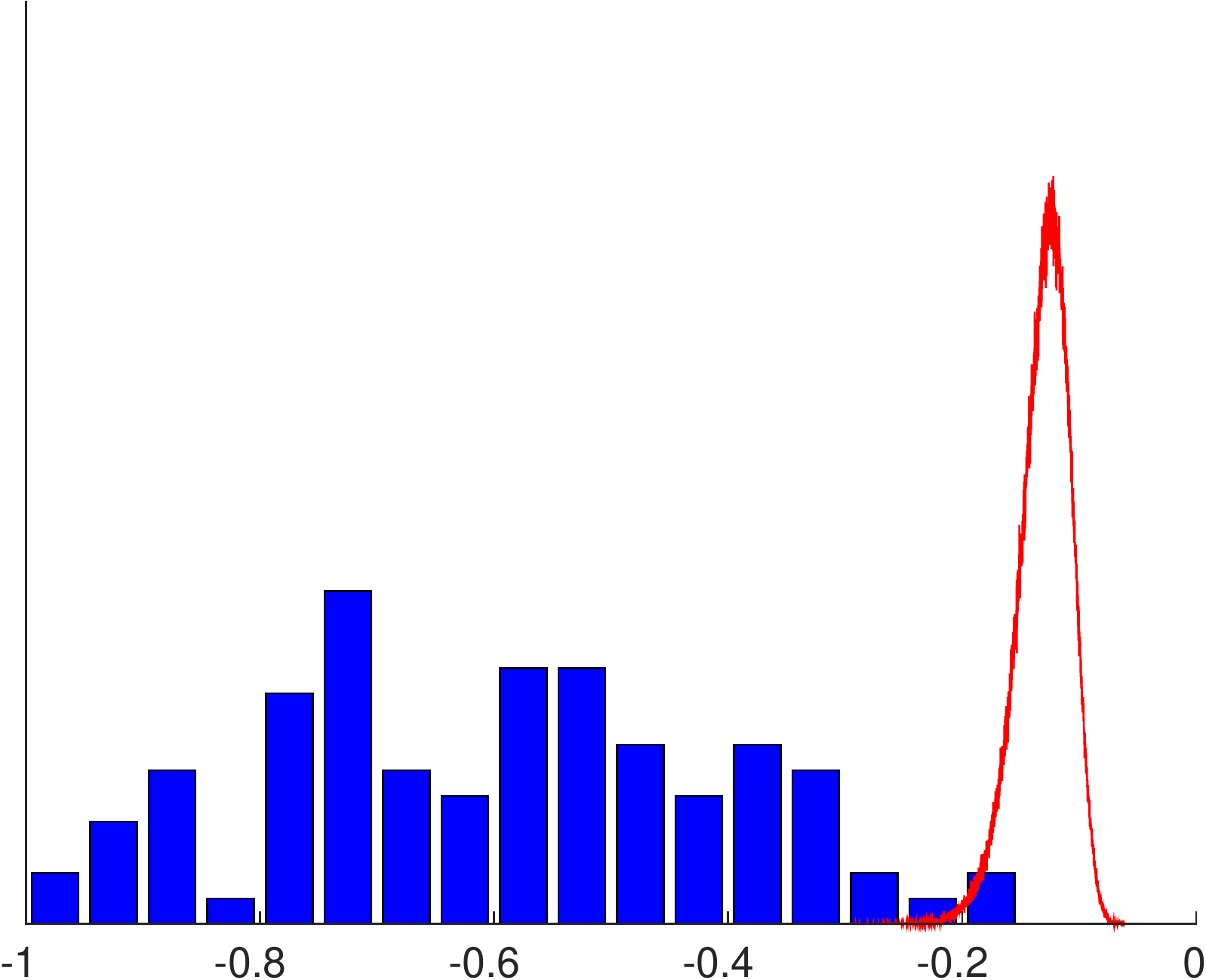}}
\subfigure[conv2]{\includegraphics[width=0.19\textwidth]{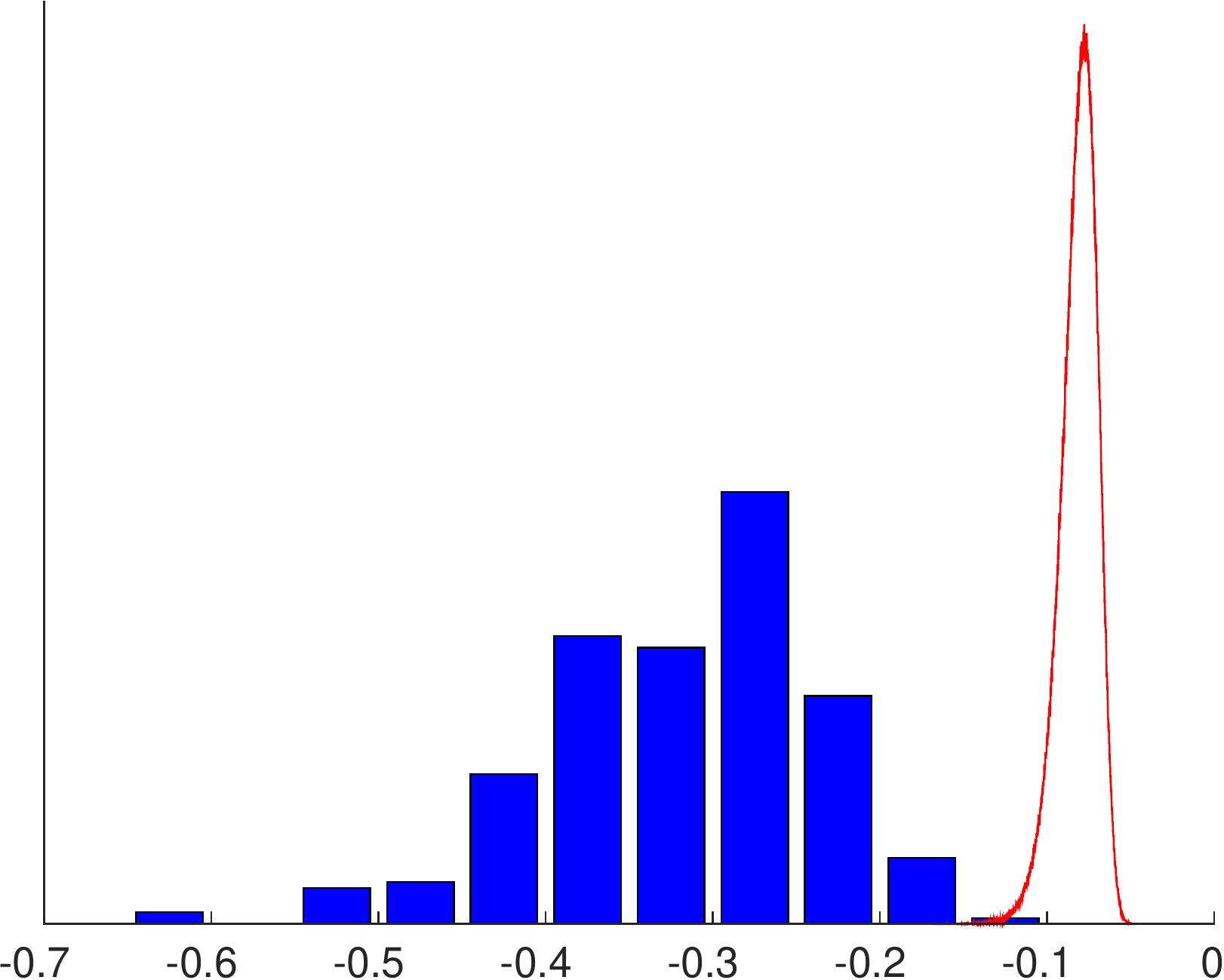}}
\subfigure[conv3]{\includegraphics[width=0.19\textwidth]{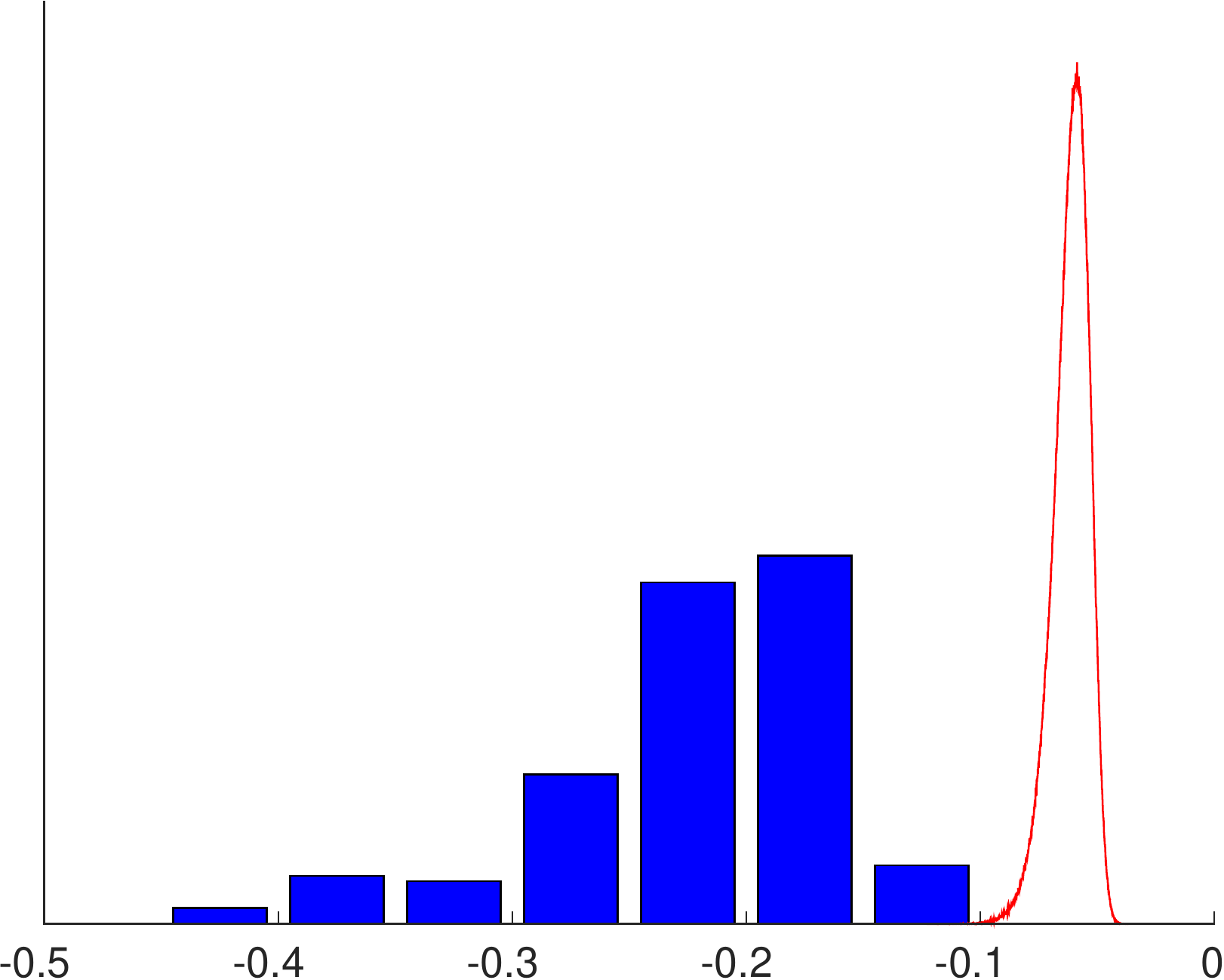}}
\subfigure[conv4]{\includegraphics[width=0.19\textwidth]{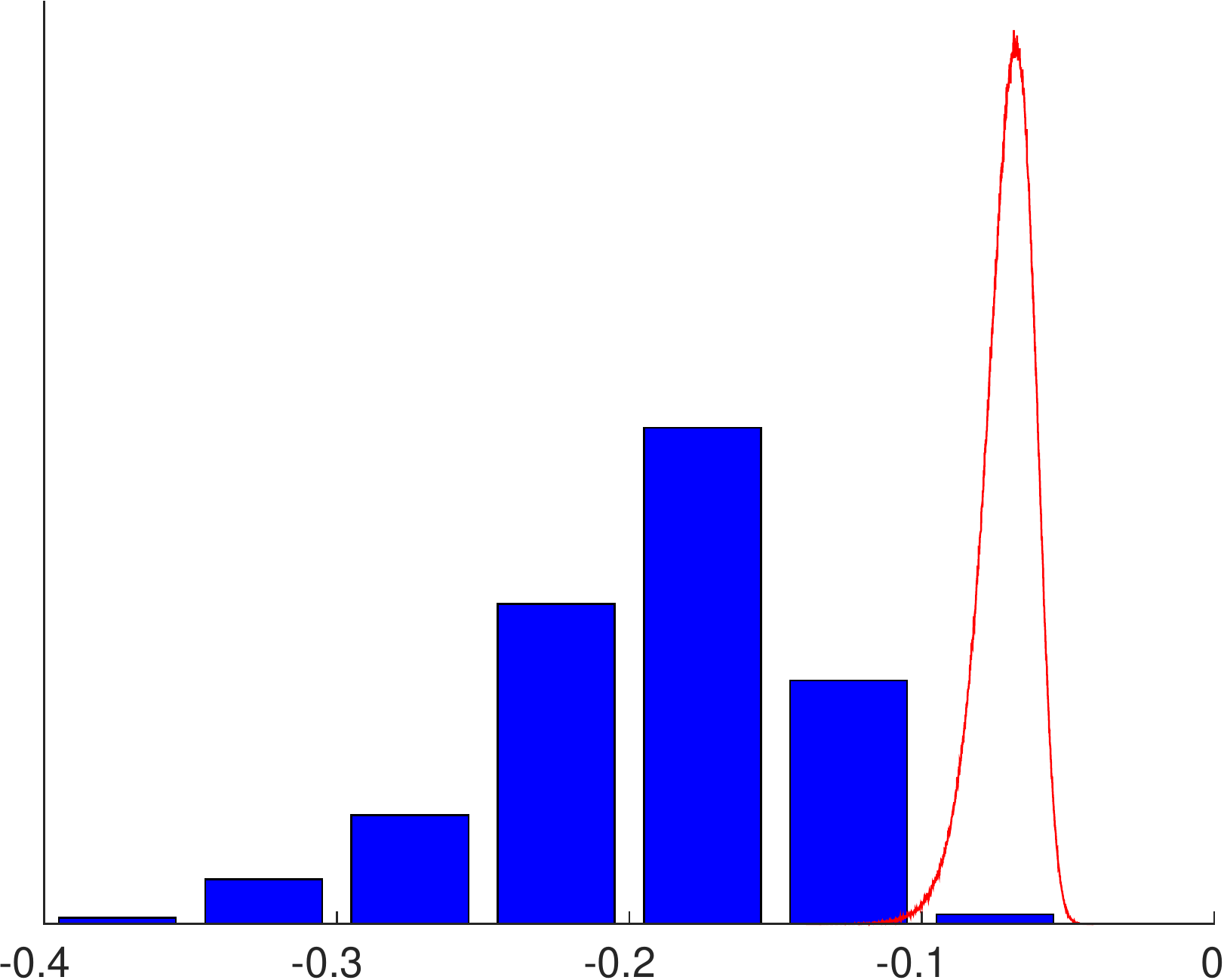}}
\subfigure[conv5]{\includegraphics[width=0.19\textwidth]{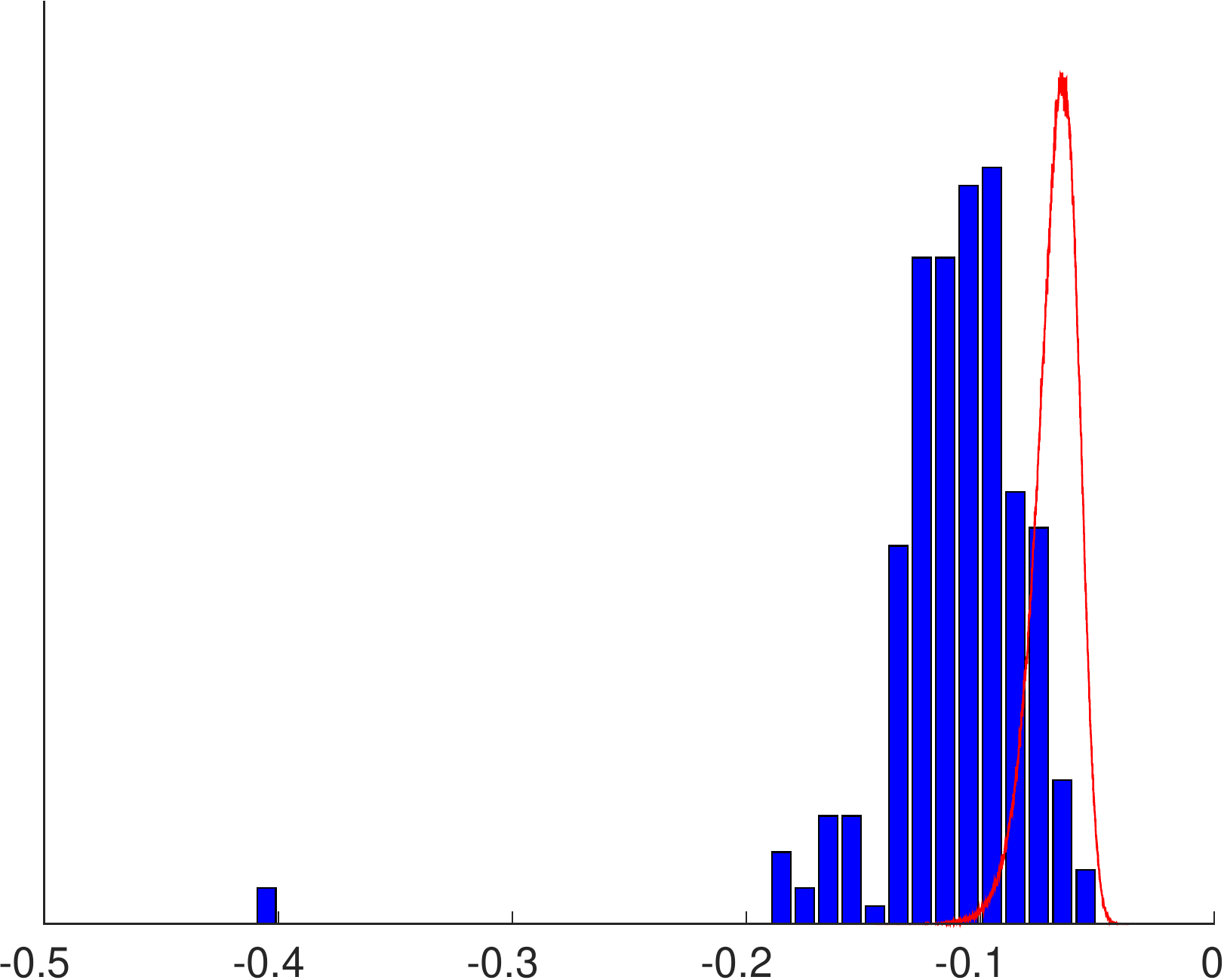}}
\vspace*{-0.15in}
\caption{\textbf{Histograms of $\mu^r$(red) and $\mu^w$(blue) for AlexNet}. 
Recall that for a set of unit length filters $\{\phi_i\}$, we define $\mu^{\phi}_i = \langle \phi_i, \bar{\phi}_i  \rangle$ where $\bar{\phi}_i$ is the pairing filter of $\phi_i$.
For conv1 layer, the distribution of $\mu^w$ (from the AlexNet filters) is negatively centered, which significantly differs from that of $\mu^r$ (from random filters), whose center is very close to zero. 
The center gradually shifts towards zero when going deeper into the network. 
}\label{fig:dist_Alex}
\vspace*{-0.15in}
\end{figure*}

%
To better comprehend the internal operations of CNNs, we investigate the well-known AlexNet~\citep{krizhevsky2012imagenet} and thereafter discover that the network learns highly negatively-correlated pairs of filters for the first few convolution layers.
Following our preliminary findings, we hypothesize that the lower convolution layers of AlexNet learn redundant filters to extract both positive and negative phase information of an input signal (Section~\ref{intuition:conj}).
%
Based on the premise of our conjecture, we propose a novel, simple yet effective activation scheme called \textbf{Concatenated Rectified Linear Unit} ($\cat$).
The proposed activation scheme preserves both positive and negative phase information while enforcing non-saturated non-linearity. 
The unique nature of $\cat$ allows a mathematical characterization of convolution layers in terms of reconstruction property, which is an important indicator of how expressive and generalizable the corresponding CNN features are (Section~\ref{sec:recon}).

In experiments, we evaluate the CNN models with $\cat$ and make a comparison to models with $\relu$ and Absolute Value Rectification Units ($\abs$)~\cite{jarrett2009best} on benchmark object recognition datasets, such as CIFAR-10/100 and ImageNet (Section~\ref{sec:experiments}). 
%
%
We demonstrate that simply replacing $\relu$ with $\cat$ for the lower convolution layers of an existing state-of-the-art CNN architecture yields a substantial improvement in classification performance. 
%
In addition, $\cat$ allows to attain notable parameter reduction without sacrificing classification performance when applied appropriately.

We analyze our experimental results from several viewpoints, such as regularization (Section~\ref{sec:regularization}) and invariant representation learning (Section~\ref{sec:invariant}). 
Retrospectively, we provide empirical evaluations on the reconstruction property of $\cat$ models; we also confirm that by integrating $\cat$, the original ``pair-grouping" phenomenon vanishes as expected (Section~\ref{sec:feature}). 
%
%
%
%
%
Overall, our results suggest that by better understanding the nature of CNNs, we are able to realize their higher potential with a simple modification of the architecture. 

\cutsectionup
\section{CRelu and Reconstruction Property}
\subsection{Conjecture on Convolution Layers}\label{intuition:conj}
In our initial exploration of classic CNNs trained on natural images such as AlexNet~\cite{krizhevsky2012imagenet}, we noted a curious property of the first convolution layer filters: \emph{these filters tend to form ``pairs''}. 
%
More precisely, assuming unit length vector for each filter $\phi_{i}$, we define a \emph{pairing filter} of $\phi_{i}$ in the following way:
$\bar{\phi}_i = \argmin_{\phi_j} \langle \phi_i,\phi_j \rangle.$
We also define their cosine similarity $\mu^{\phi}_i = \langle \phi_i, \bar{\phi}_i  \rangle$.
%
%

%
In Figure~\ref{fig:CNN_pairs_Alex}, we show each normalized filter of the first convolution layer from AlexNet with its pairing filter. 
Interestingly, they appear surprisingly opposite to each other, i.e., for each filter, there does exist another filter that is almost on the opposite phase.
Indeed, AlexNet employs the popular non-saturated activation function, \emph{Rectified Linear Unit} ($\relu$)~\cite{nair2010rectified}, which zeros out negative values and produces sparse activation. 
As a consequence, if both the positive phase and negative phase along a specific direction participate in representing the input space, the network then needs to learn two linearly dependent filters of both phases. 

To systematically study the pairing phenomenon in higher layers, we graph the histograms of $\nmu^w_i$'s for conv1-conv5 filters from AlexNet in Figure~\ref{fig:dist_Alex}.
%
For comparison, we generate random Gaussian filters $r_{i}$'s of unit norm\footnote{We sample each entry from standard normal distribution independently and normalize the vector to have unit $l^2$ norm.} and plot the histograms of $\nmu^r_i$'s together.
For conv1 layer, we observe that the distribution of $\nmu^w_i$ is negatively centered; by contrast, the mean of $\nmu^r_i$ is only slightly negative with a small standard deviation.
Then the center of $\nmu^w_i$ shifts towards zero gradually when going deeper into the network.
This implies that convolution filters of the lower layers tend to be paired up with one or a few others that represent their opposite phase, while the phenomenon gradually lessens as they go deeper.
%

%
Following these observations, we hypothesize that despite $\relu$ erasing negative linear responses, \emph{the first few convolution layers of a deep CNN manage to capture both negative and positive phase information through learning pairs or groups of negatively correlated filters}.
This conjecture implies that there exists a redundancy among the filters from the lower convolution layers. 

In fact, for a very special class of deep architecture, the invariant scattering convolutional network~\cite{bruna2013invariant}, it is well-known that its set of convolution filters, which are wavelets, is overcomplete in order to be able to fully recover the original input signals.
On the one hand, similar to $\relu$, each individual activation within the scattering network preserves partial information of the input.
On the other hand, different from $\relu$ but more similar to $\abs$, scattering network activation preserves the energy information, i.e., keeping the modulus of the responses but erasing the phase information; $\relu$ from a generic CNN, as a matter of fact, retains the phase information but eliminates the modulus information when the phase of a response is negative. 
%
In addition, while the wavelets for scattering networks are manually engineered, convolution filters from CNNs must be learned, which makes the rigorous theoretical analysis challenging.  
%
%
%

%
%
Now suppose we can leverage the pairing prior and design a method to explicitly allow both positive and negative activation, then we will be able to alleviate the redundancy among convolution filters caused by $\relu$ non-linearity and make more efficient use of the trainable parameters.
To this end, we propose a novel activation scheme, \emph{Concatenated Rectified Linear Units}, or $\cat$.
It simply makes an identical copy of the linear responses after convolution, negate them, concatenate both parts of activation, and then apply $\relu$ altogether. 
More precisely, we denote $\relu$ as $[\cdot]_+ \triangleq  \max(\cdot, 0)$, and define $\cat$ as follows:
\begin{Definition}\label{crelu}
$\cat$ activation, denoted by $\rho_c: \R \to \R^2$, is defined as follows: $\forall x\in \R, \rho_c(x) \triangleq ([x]_+, [-x]_+)$.  
\end{Definition}
\cutsectiondown
%
The rationale of our activation scheme is to allow a filter to be activated in both positive and negative direction while maintaining the same degree of non-saturated non-linearity.

A resembling method, namely soft-thresholding~\cite{coates}, has been applied as a separate feature encoding step to unsupervised dictionary learning to generate more separable features for linear SVM classifier.
%
%
Concurrent to our work, other research groups also conducted related studies independently and presented as MaxMin scheme~\cite{ICIP}, ON/OFF ReLU~\cite{kim2015convolutional}, or as Antirectifier~\cite{Keras}.\footnote{The antirectifier (\url{https://github.com/fchollet/keras/blob/master/examples/antirectifier.py}) has slightly different formulation to ours as it involves a few preprocessing steps such as mean subtraction and normalization before concatenated rectification.}
In comparison, we provide comprehensive experiments on large-scale datasets using deeper network architectures as well as qualitative analysis.


%
%
%
%
An alternative way to allow negative activation is to employ the broader class of non-saturated activation functions including Leaky $\relu$ and its variants~\cite{maas2013rectifier,xu2015empirical}. 
%
Leaky $\relu$ assigns a small slope to the negative part instead of completely dropping it.  
%
These activation functions share similar motivation with $\cat$ in the sense that they both tackle the two potential problems caused by the hard zero thresholding: (1) the weights of a filter will not be adjusted if it is never activated, and (2) truncating all negative information can potentially hamper the learning.
However, $\cat$ is based on an activation \emph{scheme} rather than a \emph{function}, which fundamentally differentiates itself from Leaky $\relu$ or other variants.
In our version, we apply $\relu$ after separating the negative and positive part to compose $\cat$, but it is not the only feasible non-linearity. 
%
For example, $\cat$ can be combined with other activation functions, such as Leaky $\relu$, to add more diversity to the architecture. 
%

Another natural analogy to draw is between $\cat$ and $\abs$, where the latter one only preserves the modulus information but discard the phase information, similar to the scattering network. 
$\abs$ has not been widely used recently for the CNN models due to its suboptimal empirical performance.
We confirm this common belief in the matter of large-scale image recognition task (Section~\ref{sec:experiments}) and conclude that modulus information alone does not suffice to produce state-of-the-art deep CNN features. 
\subsection{Reconstruction Property}\label{sec:recon}
A notable property of $\cat$ is its \emph{information preservation nature}: $\cat$ conserves both negative and positive linear responses after convolution.
A direct consequence of information preserving is the reconstruction power of the convolution layers equipped with $\cat$. 
%

Reconstruction property of a CNN implies that its features are representative of the input data.
%
This aspect of CNNs has gained interest recently: \citet{mahendran2014understanding} invert CNN features back to the input under simple natural image priors; \citet{zhao2015stacked} stack autoencoders with reconstruction objective to build better classifiers. 
%
%
\citet{bruna2013signal} theoretically investigate general conditions under which the max-pooling layer followed by $\relu$ is injective and measure stability of the inverting process by computing the Lipschitz lower bound.
However, their bounds are non-trivial only when the number of filters significantly outnumbers the input dimension, which is not realistic. 
%
%

In our case, it becomes more straightforward to analyze the reconstruction property since $\cat$ preserves all the information after convolution.
The rest of this section mathematically characterizes the reconstruction property of a single convolution layer followed by $\cat$ with or without max-pooling layer.
%
%

%

%
We first analyze the reconstruction property of convolution followed by $\cat$ without max-pooling.
This case is directly pertinent as deep networks replacing max-pooling with stride has become more prominent in recent studies~\cite{springenberg2014striving}.
The following proposition states that the part of an input signal spanned by the shifts of the filters is well preserved.
%
%
%
%
\begin{proposition}\label{recon1}
Let $x\in\mathbb{R}^{D}$ be an input vector\footnote{For clarity, we assume the input signals are vectors (1D) rather than images (2D); however, similar analysis can be done for 2D case.} and $W$ be the $D$-by-$K$ matrix whose columns vectors are composed of $w_i \in\mathbb{R}^l, i = 1, \ldots ,K$ convolution filters. Furthermore, let $x = x' +  (x - x')$, where $x' \in\mathrm{range}(W)$ and $x-x' \in \mathrm{ker}(W)$. Then we can reconstruct $x'$ with $\cnn(x)$, where $\cnn(x)\triangleq \cat\left(W^{T}x\right)$.
\end{proposition}
See Section~\ref{nonmax} in the supplementary materials for proof.
%

%
Next, we add max-pooling into the picture.
To reach a non-trivial bound, we need additional constraints on the input space.
Due to space limit, we carefully explain the constraints and the theoretical consequence in Section~\ref{maxpooling} of the supplementary materials. 
We will revisit this subject after the experiment section (Section~\ref{sec:feature}). 

\begin{table*}[t]
\caption{\textbf{Test set recognition error rates on CIFAR-10/100.} We compare the performance of $\relu$ models (baseline) and $\cat$ models with different model capacities: ``double" refers to the models that double the number of filters and ``half" refers to the models that halve the number of filters. The error rates are provided in multiple ways, such as ``Single", ``Average" (with standard error), or ``Vote", based on cross-validation methods. We also report the corresponding train error rates for the Single model. The number of model parameters are given in million. Please see the main text for more details about model evaluation.\label{tab:cifar_results}}
\vspace{0.05in}
\centering
\small
\begin{tabular}{c|c|c|c|c|c|c|c|c|c} 
\hline
\multirow{3}{*}{Model} & \multicolumn{4}{c|}{CIFAR-10} & \multicolumn{4}{c|}{CIFAR-100} & \multirow{3}{*}{params.}\\ \cline{2-9}
 & \multicolumn{2}{c|}{Single} & \multirow{2}{*}{Average} & \multirow{2}{*}{Vote} & \multicolumn{2}{c|}{Single} & \multirow{2}{*}{Average} & \multirow{2}{*}{Vote} &  \\ \cline{2-3}\cline{6-7}
 & train & test &  &  & train & test &  &  &  \\ \hline
Baseline & 1.09 & 9.17 & 10.20{\scriptsize$\pm 0.09$} & 7.55 & 13.68 & 36.30 & 38.52{\scriptsize$\pm 0.12$} & 31.26 & $1.4$M \\
{\small$+$} (double) & 0.47 & 8.65 & 9.87{\scriptsize$\pm 0.09$} & 7.28 & 6.03 & 34.77 & 36.73{\scriptsize$\pm 0.15$} & 28.34 & $5.6$M \\ \hline
$\abs$ & 4.10 & {\bf 8.32} & 10.26{\scriptsize$\pm 0.10$} & 7.76 & 19.35 & 35.00 & 37.24{\scriptsize$\pm 0.20$} & 29.77 & $1.4$M \\ \hline
$\cat$ & 4.23 & 8.43 & \bf{9.39}{\scriptsize$\pm 0.11$} & \bf{7.09} & 14.25 & {\bf 31.48} & \bf{33.76}{\scriptsize$\pm 0.12$} & {\bf 27.60} & $2.8$M \\ 
{\small$+$} (half) & 4.73 & {\bf 8.37} & \bf{9.44}{\scriptsize$\pm 0.09$} & {\bf 7.09} & 21.01 & 33.68 & 36.20{\scriptsize$\pm 0.18$} & 29.93 & $0.7$M \\ 
\hline
\end{tabular}
\vspace{-0.1in}
\end{table*}

\cutsectionup
\section{Benchmark Results}
\label{sec:experiments}
%
We evaluate the effectiveness of the $\cat$ activation scheme on three benchmark datasets: CIFAR-10, CIFAR-100~\cite{krizhevsky2009learning} and ImageNet~\cite{deng2009imagenet}.
To directly assess the impact of $\cat$, we employ existing CNN architectures with $\relu$ that have already shown a good recognition baseline and demonstrate improved performance on top by replacing $\relu$ into $\cat$.
Note that the models with $\cat$ activation don't need significant hyperparameter tuning from the baseline $\relu$ model, and in most of our experiments, we only tune dropout rate while other hyperparameters (e.g., learning rate, mini-batch size) remain the same.
We also replace $\relu$ with $\abs$ for comparison with $\cat$.
The details of network architecture are in Section~\ref{sec:model} of the supplementary materials.

\cutsubsectionup
\subsection{CIFAR-10 and CIFAR-100}
\label{sec:exp-cifar}
The CIFAR-10 and 100 datasets~\cite{krizhevsky2009learning} each consist of $50,000$ training and $10,000$ testing examples of $32\times 32$ images evenly drawn from 10 and 100 classes, respectively.
We subtract the mean and divide by the standard deviation for preprocessing and use random horizontal flip for data augmentation.
%
 
%
We use the ConvPool-CNN-C model~\citep{springenberg2014striving} as our baseline model, which is composed of convolution and pooling followed by $\relu$ without fully-connected layers.
This baseline model serves our purpose well since it has clearly outlined network architecture only with convolution, pooling, and $\relu$. 
It has also shown competitive recognition performance using a fairly small number of model parameters. 
%
%

First, we integrate $\cat$ into the baseline model by simply replacing $\relu$ while keeping the number of convolution filters the same.
This doubles the number of output channels at each convolution layer and the total number of model parameters is doubled.
To see whether the performance gain comes from the increased model capacity, we conduct additional experiments with the baseline model while doubling the number of filters and the $\cat$ model while halving the number of filters.
We also evaluate the performance of the $\abs$ model while keeping the number of convolution filters the same as the baseline model.

Since the datasets don't provide pre-defined validation set, we conduct two different cross-validation schemes:
\begin{enumerate}[itemsep=5pt,nolistsep]
\item{``Single": we hold out a subset of training set for initial training and retrain the network from scratch using the whole training set until we reach at the same loss on a hold out set~\cite{goodfellow2013maxout}. For this case, we also report the corresponding train error rates.}
\item{$10$-folds: we divide training set into $10$ folds and do validation on each of $10$ folds while training the networks on the rest of $9$ folds. The mean error rate of single network (``Average") and the error rate with model averaging of $10$ networks (``Vote") are reported.}
\end{enumerate}
The recognition results are summarized in Table~\ref{tab:cifar_results}.
On CIFAR-10, we observe significant improvement with the $\cat$ activation over $\relu$.
Especially, $\cat$ models consistently improve over $\relu$ models with the same number of neurons (or activations) while reducing the number of model parameters by half (e.g., $\cat$ + half model and the baseline model have the same number of neurons while the number of model parameters are $0.7$M and $1.4$M, respectively).
On CIFAR-100, the models with larger capacity generally improve the performance for both activation schemes. 
Nevertheless, we still find a clear benefit of using $\cat$ activation that shows significant performance gain when it is compared to the model with the same number of neurons, i.e., half the number of model parameters.
One possible explanation for the benefit of using $\cat$ is its regularization effect, as can be confirmed in Table~\ref{tab:cifar_results} that the $\cat$ models showed significantly lower gap between train and test set error rates than those of the baseline $\relu$ models.

To our slight surprise, $\abs$ outperforms the baseline $\relu$ model on CIFAR-100 with respect to all evaluation metrics and on CIFAR-10 with respect to single-model evaluation. 
%
It also reaches promising single-model recognition accuracy compared to $\cat$ on CIFAR-10; however, when averaging or voting across 10-folds validation models, $\abs$ becomes clearly inferior to $\cat$.  
\paragraph{Experiments on Deeper Networks.} 
We conduct experiments with very deep CNN that has a similar network architecture to the VGG network~\cite{simonyan2014very}.
Specifically, we follow the model architecture and training procedure in~\citet{Zblog}.
%
%
%
Besides the convolution and pooling layers, this network contains batch normalization~\cite{ioffe2015batch} and fully connected layers. 
Due to the sophistication of the network composition which may introduce complicated interaction with $\cat$, we only integrate $\cat$ into the first few layers.
Similarly, we subtract the mean and divide by the standard deviation for preprocessing and use horizontal flip and random shifts for data augmentation. 

In this experiment\footnote{We attempted to replace $\relu$ with $\abs$ on various layers but we observed significant performance drop with $\abs$ non-linearity when used for deeper networks.}, we gradually replace $\relu$ after the first, third, and the fifth convolution layers\footnote{Integrating $\cat$ into the second or fourth layer before max-pooling layers did not improve the performance.} with $\cat$ while halving the number of filters, resulting in a reduced number of model parameters. 
%
%
We report the test set error rates using the same cross-validation schemes as in the previous experiments. 
As shown in Table~\ref{tab:VGGresult}, there is substantial performance gain in both datasets by replacing $\relu$ with $\cat$.
Overall, the proposed $\cat$ activation improves the performance of the state-of-the-art VGG network significantly, achieving highly competitive error rates to other state-of-the-art methods, as summarized in Table~\ref{tab:cifar_compare}.

\begin{table}[t]
\caption{\textbf{Test set recognition error rates on CIFAR-10/100 using deeper networks.} We gradually apply $\cat$ to replace $\relu$ after conv1, conv3, and conv5 layers of the baseline VGG network while halving the number of convolution filters.\label{tab:VGGresult}}
\vspace{0.05in}
\centering
\small
\begin{tabular}{c|c|c|c} 
\hline
\multicolumn{4}{c}{CIFAR-10}\\ \hline
Model & Single & Average & Vote  \\ \hline
VGG & 6.35 &    6.90{\scriptsize$\pm 0.03$} & 5.43  \\ \hline
(conv1) & 6.18  &    \bf{6.45}{\scriptsize$\pm 0.05$} & 5.22  \\ \hline
(conv1,3) &  \bf{5.94} &    \bf{6.45}{\scriptsize$\pm 0.02$} & \bf{5.09} \\  \hline
(conv1,3,5) & 6.06 &    \bf{6.45}{\scriptsize$\pm 0.07$} & 5.16 \\ \hline\hline
\multicolumn{4}{c}{CIFAR-100}\\ \hline
Model & Single & Average & Vote \\ \hline
VGG & 28.99 &   30.27{\scriptsize$\pm 0.09$} & 26.85  \\ \hline
(conv1) & 27.29 &   28.43{\scriptsize$\pm 0.11$} & 24.67 \\ \hline
(conv1,3) & 26.52 &  27.79{\scriptsize$\pm 0.08$} & 23.93  \\ \hline
(conv1,3,5) & \bf{26.16}  &  \bf{27.67}{\scriptsize$\pm 0.07$} & \bf{23.66}  \\ \hline
\end{tabular}
\vspace{-0.1in}
\end{table}

\cutsubsectionup
\subsection{ImageNet}
\label{sec:exp-imagenet}
To assess the impact of $\cat$ on large scale dataset, we perform experiments on ImageNet dataset~\cite{deng2009imagenet}\footnote{We used a version of ImageNet dataset for ILSVRC 2012.}, which contains about $1.3$M images for training and $50,000$ for validation from $1,000$ object categories.
%
For preprocessing, we subtract the mean and divide by the standard deviation for each input channel, and follow the data augmentation as described in~\cite{krizhevsky2012imagenet}.
%
%

%
We take the All-CNN-B model~\citep{springenberg2014striving} as our baseline model. 
The network architecture of All-CNN-B is similar to that of AlexNet~\cite{krizhevsky2012imagenet}, where the max-pooling layer is replaced by convolution with the same kernel size and stride, the fully connected layer is replaced by $1\times1$ convolution layers followed by average pooling, and the local response normalization layers are discarded.
In sum, the layers other than convolution layers are replaced or discarded and finally the network consists of convolution layers only. 
%
We choose this model since it reduces the potential complication introduced by $\cat$ interacting with other types of layers, such as batch normalization or fully connected layers.
%
%

\begin{table}[t]
\caption{\textbf{Comparison to other methods on CIFAR-10/100.}\label{tab:cifar_compare}} 
\vspace{0.05in}
\centering
\small
\begin{tabular}{c|c|c}
\hline
Model & CIFAR-10 & CIFAR-100  \\ \hline
\citep{rippel2015spectral} & 8.60 &  31.60   \\ \hline
\citep{snoek2015scalable} & 6.37 & 27.40   \\ \hline 
\citep{liang2015recurrent} & 7.09 & 31.75   \\ \hline 
\citep{lee2015generalizing} & 6.05  & 32.37  \\ \hline 
\citep{srivastava2015training} & 7.60 & 32.24   \\ \thickhline
VGG & 5.43 & 26.85 \\ \hline
VGG + $\cat$ & \textbf{5.09}  & \textbf{23.66}   \\ \hline 
\end{tabular}
\vspace{-0.05in}
\end{table}

We gradually integrate more convolution layers with $\cat$ (e.g., conv1--4, conv1--7, conv1--9), while keeping the same number of filters. 
These models contain more parameters than the baseline model.
%
We also evaluate two models where one replaces all $\relu$ layers into $\cat$ and the other conv1,conv4 and conv7 only, where both models reduce the number of convolution layers before $\cat$ by half.
Hence, these models contain fewer parameters than the baseline model.
For comparison, $\abs$ models are also constructed by gradually replacing $\relu$ in the same manner as the $\cat$ experiments (conv1--4, conv1--7, conv1--9). 
%
%
The network architectures and the training details are in Section~\ref{sec:model} and Section~\ref{sec:detail-training} of the supplementary materials. 
\begin{figure*}[t]
\centering
\subfigure[conv1]{\includegraphics[width=0.22\textwidth]{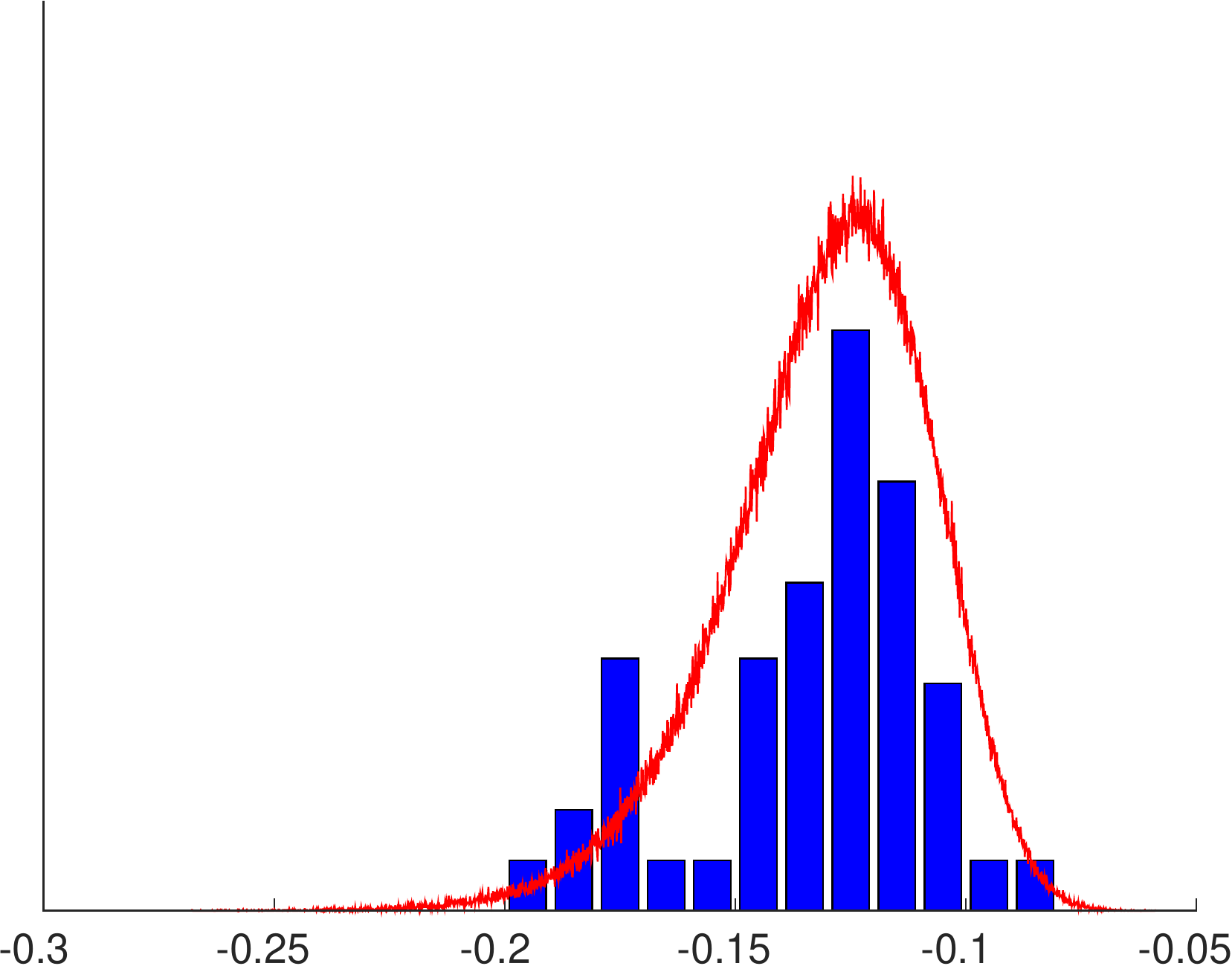}}\hspace{0.07in}
\subfigure[conv2]{\includegraphics[width=0.22\textwidth]{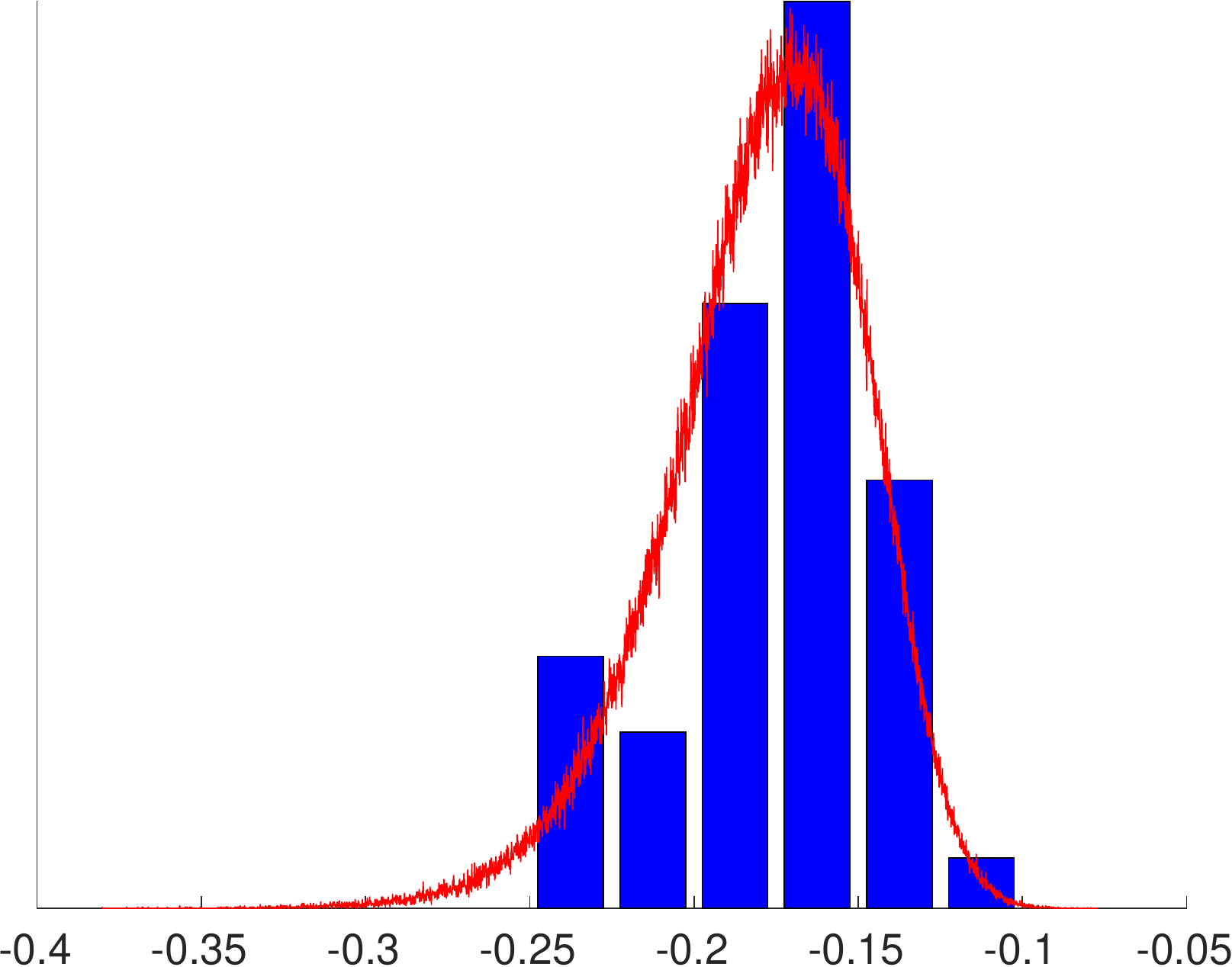}}\hspace{0.07in}
\subfigure[conv3]{\includegraphics[width=0.22\textwidth]{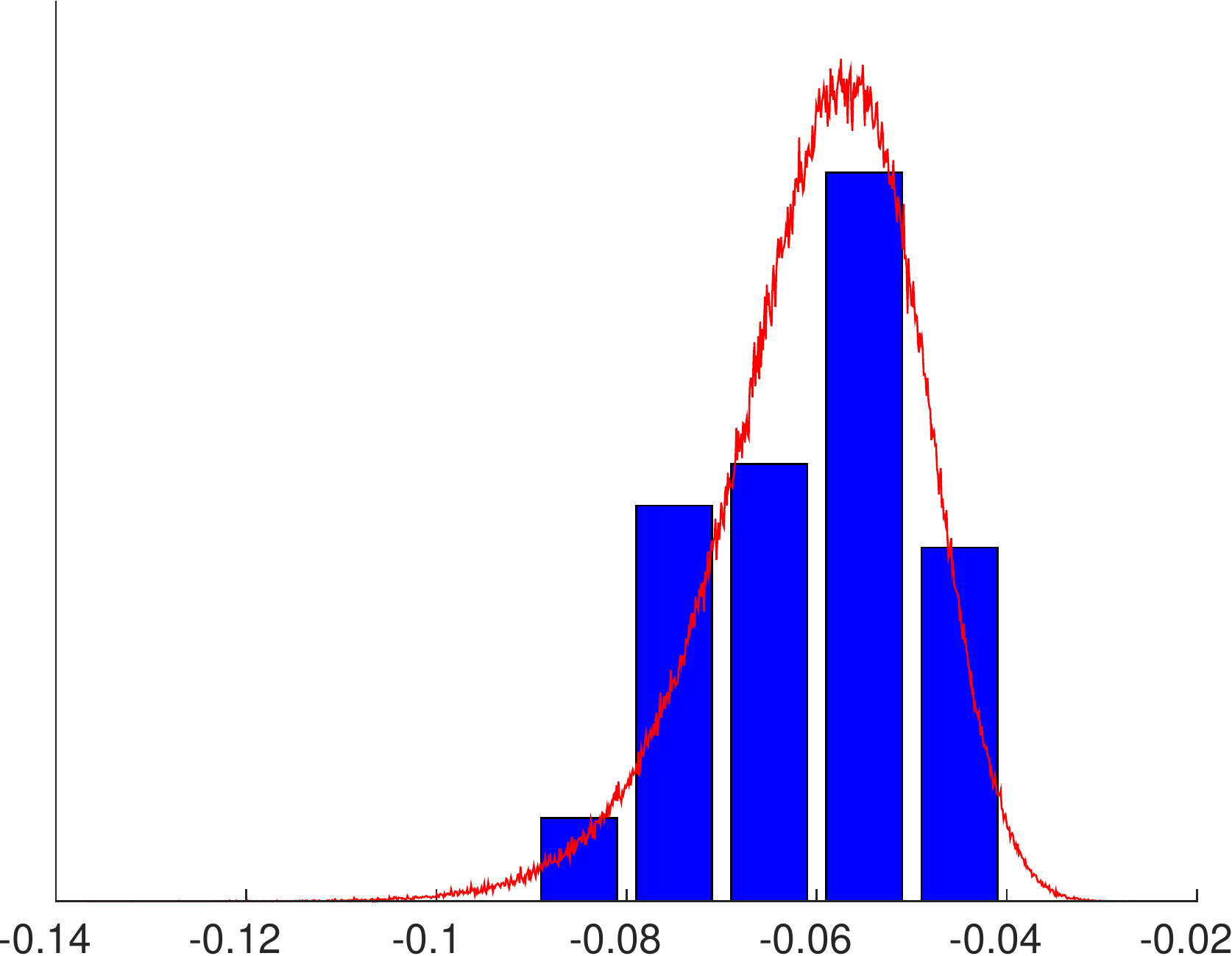}}\hspace{0.07in}
\subfigure[conv4]{\includegraphics[width=0.22\textwidth]{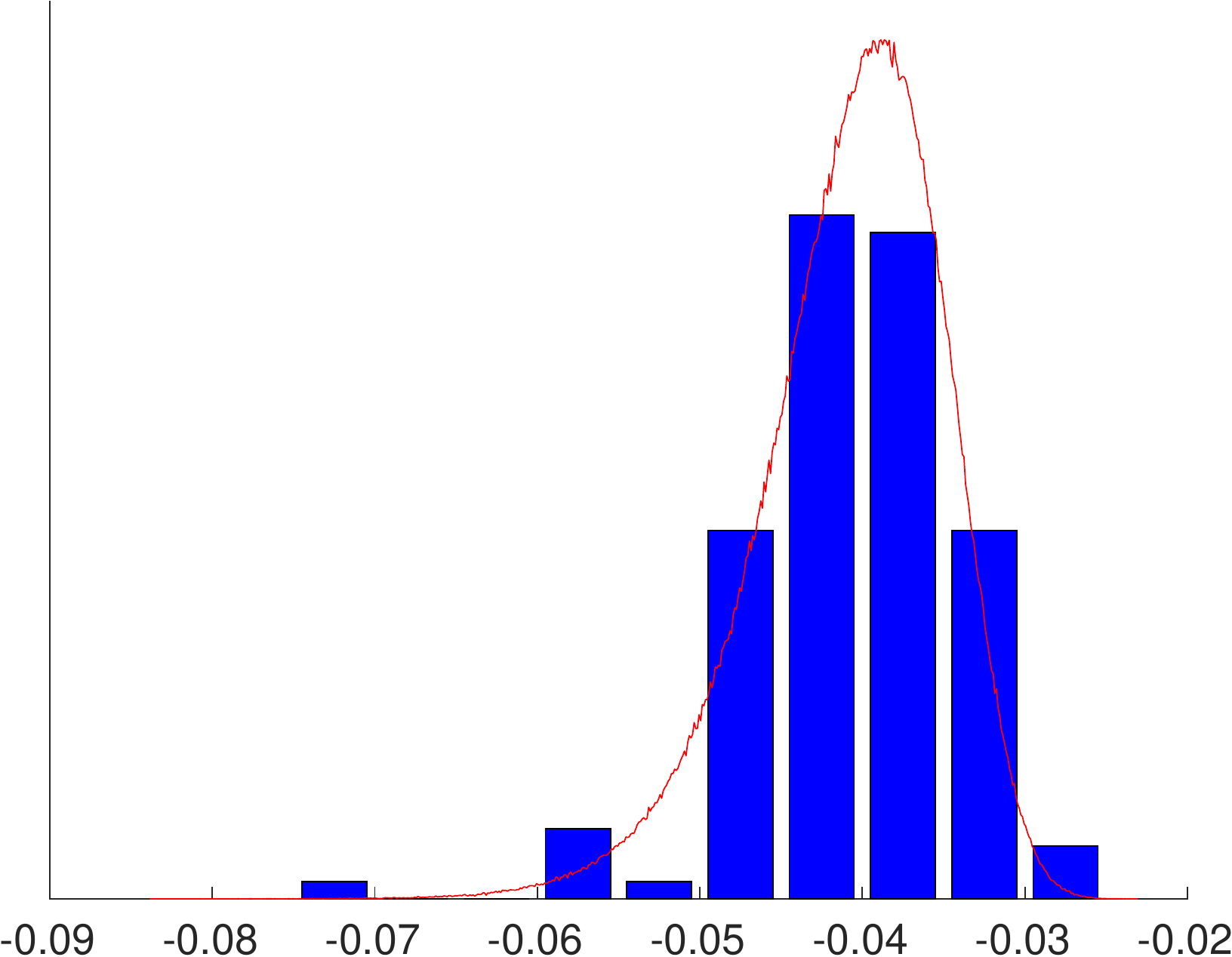}}
\vspace{-0.15in}
\caption{\textbf{Histograms of $\mu^r$(red) and $\mu^w$(blue) for $\cat$ model on ImageNet.} The two distributions align with each other for all conv1-conv4 layers--as we expected, the pairing phenomenon is not present any more after applying the $\cat$ activation scheme.
}\label{fig:dist_allconc3}
\vspace{-0.25in}
\end{figure*}

The results are provided in Table~\ref{tab:AllConvResult}. 
We report the top-1 and top-5 error rates with center crop only and by averaging scores over $10$ patches from the center crop and four corners and with horizontal flip~\cite{krizhevsky2012imagenet}.
Interestingly, integrating $\cat$ to conv1-4 achieves the best results, whereas going deeper with higher model capacity does not further benefit the classification performance.
In fact, this parallels with our initial observation on AlexNet (Figure~\ref{fig:dist_Alex} in Section~\ref{intuition:conj})---there exists less ``pairing'' in the deeper convolution layers and thus there is not much gain by decomposing the phase in the deeper layers. 
$\abs$ networks exhibit the same trend but do not noticeably improve upon the baseline performance, which implies that $\abs$ is not the most suitable candidate for large-scale deep representation learning. 
Another interesting observation, which we will discuss further in Section~\ref{sec:invariant}, is that the model integrating $\cat$ into conv1, conv4 and conv7 layers also achieve highly competitive recognition results with even fewer parameters than the baseline model. 
%
%
In sum, we believe that such a significant improvement over the baseline model by simply modifying the activation scheme is a pleasantly surprising result.\footnote{We note that \citet{springenberg2014striving} reported slightly better result ($41.2\%$ top-1 error rate with center crop only) than our replication result, but still the improvement is significant.}
We also compare our best models with AlexNet and other variants in Table~\ref{tab:ImageNetResult}.
Even though reducing the number of parameters is not our primary goal, it is worth noting that our model with only $4.6$M parameters ($\cat$ + all) outperforms FastFood-32-AD (FriedNet)~\cite{yang2014deep} and Pruned AlexNet (PrunedNet)~\cite{han2015learning}, whose designs directly aim at parameter reduction. 
Therefore, besides the performance boost, another significance of $\cat$ activation scheme is in designing more parameter-efficient deep neural networks.

\begin{table}[t]
\caption{\textbf{Validation error rates on ImageNet.} We compare the performance of baseline model with the proposed $\cat$ models at different levels of activation scheme replacement. Error rates with $^\dagger$ are obtained by averaging scores from $10$ patches.\label{tab:AllConvResult}}
\vspace{0.05in}
\centering
\small
\begin{tabular}{c|c|c|c|c} 
\hline
Model & top-1 & top-5 & top-1$^\dagger$ & top-5$^\dagger$ \\ \hline
Baseline & 41.81 & 19.74  & 38.03  & 17.17 \\ \hline
$\abs$ (conv1--4) & 41.12 & 19.25 & 37.32   & 16.49 \\ \hline  
$\abs$ (conv1--7) &  42.36 & 20.05 & 38.21 & 17.42 \\ \hline
$\abs$ (conv1--9) & 43.33 & 21.05 & 39.70  & 18.39 \\ \hline
$\cat$ (conv1,4,7) & 40.45 &  18.58 & \textbf{35.70}   & \textbf{15.32} \\ \hline
$\cat$ (conv1--4) & \textbf{39.82} & \textbf{18.28} & 36.20   & 15.72 \\ \hline  
$\cat$ (conv1--7) &  39.97 & 18.33 & 36.53 & 16.01 \\ \hline
$\cat$ (conv1--9) & 40.15 & 18.58 & 36.50  & 16.14 \\ \hline
$\cat$ (all) & 40.93 & 19.39 & 37.28 & 16.72 \\ \hline
\end{tabular}
\vspace{-0.05in}
\end{table}

\begin{table}[t]
\caption{\textbf{Comparison to other methods on ImageNet.} We compare with AlexNet and other variants, such as  FastFood-32-AD (FriedNet)~\cite{yang2014deep} and pruned AlexNet (PrunedNet)~\cite{han2015learning}, which are modifications of AlexNet aiming at reducing the number of parameters, as well as All-CNN-B, the baseline model~\cite{springenberg2014striving}. Error rates with $^\dagger$ are obtained by averaging scores from $10$ patches.\label{tab:ImageNetResult}}
\vspace{0.05in}
\centering
\small
\begin{tabular}{c|c|c|c|c|c} 
\hline
Model & top-1 & top-5 & top-1$^\dagger$ & top-5$^\dagger$ & params. \\ \hline
AlexNet & 42.6 & 19.6   & 40.7 & 18.2 & $61$M \\ \hline
FriedNet & 41.93 & -- & -- & -- & $32.8$M \\ \hline
PrunedNet & 42.77 & 19.67 & -- & -- & $6.7$M \\ \hline
AllConvB & 41.81 & 19.74  & 38.03  & 17.17& $9.4$M \\ 
\thickhline
$\cat$ (all) & 40.93 & 19.39 & 37.28 & 16.72 & $\mathbf{4.7}$M \\ \hline
(conv1,4,7) & 40.45 &  18.58 & \textbf{35.70}   & \textbf{15.32} & 8.6 M \\ \hline
(conv1--4) & \textbf{39.82} & \textbf{18.28} & 36.20   & 15.72 & $10.1$M  \\ \hline
\end{tabular}
\vspace{-0.05in}
%
%
%
\end{table}

\cutsectionup
\section{Discussion}
\label{sec:discussion}
%
In this section, we discuss qualitative properties of $\cat$ activation scheme in several viewpoints, such as regularization of the network and learning invariant representation.
%
\begin{figure*}[htbp]
\centering
\subfigure[CIFAR-10]{\label{fig:cifar100_traintest}\includegraphics[height=1.5in]{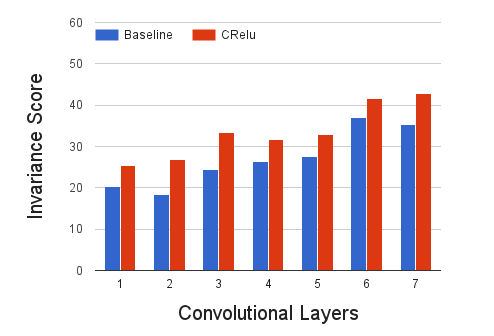}}
\subfigure[CIFAR-100]{\label{fig:cifar100_test}\includegraphics[height=1.5in]{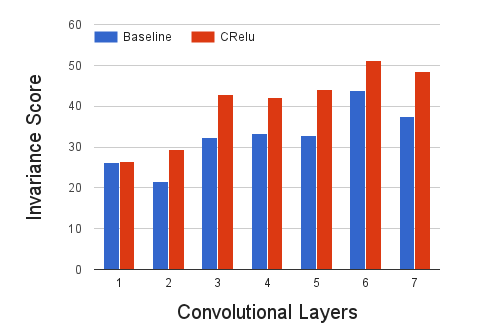}}
\subfigure[ImageNet]{\label{fig:cifar100_test}\includegraphics[height=1.5in]{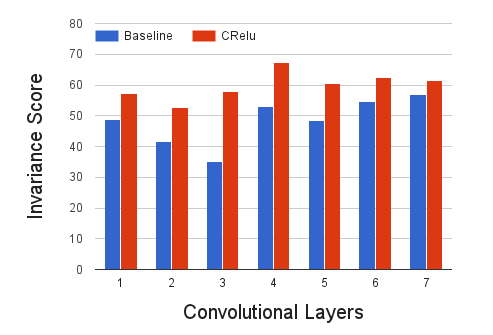}}
\vspace{-0.1in}
\caption{\textbf{Invariance Scores for $\relu$ Models vs $\cat$ Models}. The invariance scores for $\cat$ models are consistently higher than $\relu$ models. The invariance scores jump after max-pooling layers. Moreover, even though the invariance scores tend to increase along with the depth of the networks, the progression is not monotonic.\label{fig:invar_cifar}}
\vspace{-0.1in}
\end{figure*}

\cutsubsectionup
\subsection{A View from Regularization}
\label{sec:regularization}
In general, a model with more trainable parameters is more prone to overfitting.
However, somewhat counter-intuitively, for the all-conv CIFAR experiments, the models with $\cat$ display much less overfitting issue compared to the baseline models with $\relu$, even though it has twice as many parameters (Table~\ref{tab:cifar_results}). 
We contemplate that keeping both positive and negative phase information makes the training more challenging, and such effect has been leveraged to better regularize deep networks, especially when working on small datasets.
%
%

%
Besides the empirical evidence, we can also describe the regularization effect by deriving a Rademacher complexity bound for the $\cat$ layer followed by linear transformation as follows:
\begin{theorem}\label{RCBound}
Let $\mathcal{G}$ be the class of real functions $\R^{\din}\to \R$ with input dimension $\F$, that is, $\mathcal{G} = [\F]_{j=1}^{\din}$. Let $\mathcal{H}$ be a linear transformation function from $\R^{2\din}$ to $\R$, parametrized by $W$, where $\|W\|_2\leq B$. Then, we have
\[\eR (\mathcal{H}\circ \rho_c \circ \mathcal{G}) \leq \sqrt{\din} B \eR (\F).\]
\end{theorem} 
%
%
%
The proof is in Section~\ref{sec:proof_complex} of the supplementary materials. 
Theorem~\ref{RCBound} says that the complexity bound of $\cat$ + linear transformation is the same as that of $\relu$ + linear transformation, which is proved by \citet{wan2013regularization}.
%
In other words, although the number of model parameters are doubled by $\cat$, the model complexity does not necessarily increase.
%
%
%
%
%
%
%

\begin{table}[t]
\caption{\textbf{Correlation Comparison.} The averaged correlation between the normalized positive-negative-pair (pair) outgoing weights and the normalized unmatched-pair (non-pair) outgoing weights are both well below $1$ for all layers, indicating that the pair outgoing weights are capable of imposing diverse non-linear manipulation separately on the positive and negative components. \label{tab:correlation}}
\vspace{0.05in}
\centering
\small
\begin{tabular}{c|c|c} 
\hline
\multicolumn{3}{c}{ImageNet Conv1--7 $\cat$ Model}\\ \hline
layer & pair & non-pair \\ \hline
conv1 & 0.372 {\scriptsize$\pm 0.372$}  & 0.165 {\scriptsize$\pm 0.154$}   \\ \hline
conv2& 0.180 {\scriptsize$\pm 0.149$} & 0.157 {\scriptsize$\pm 0.137$}  \\ \hline 
conv3 & 0.462 {\scriptsize$\pm 0.249$} & 0.120 {\scriptsize$\pm 0.120$}   \\ \hline
conv4 & 0.175 {\scriptsize$\pm 0.146$}& 0.119  {\scriptsize$\pm 0.100$}   \\ \hline
conv5& 0.206 {\scriptsize$\pm 0.136$}& 0.105 {\scriptsize$\pm 0.093$}  \\ \hline 
conv6& 0.256 {\scriptsize$\pm 0.124$}& 0.086 {\scriptsize$\pm 0.080$}  \\ \hline 
conv7& 0.131 {\scriptsize$\pm 0.122$}& 0.080 {\scriptsize$\pm 0.070$}  \\ \hline 
\end{tabular}
\vspace{-0.15in}
\end{table}
\cutsectiondown

\cutsubsectionup
\subsection{Towards Learning Invariant Features}
\label{sec:invariant}
%
%
%

%
%
%
%
%
%
We measure the invariance scores using the evaluation metrics from~\citep{goodfellow2009measuring} and draw another comparison between the $\cat$ models and the $\relu$ models.
For a fair evaluation, we compare all 7 conv layers from all-conv $\relu$ model with those from all-conv $\cat$ model trained on CIFAR-10/100. 
In the case of ImageNet experiments, we choose the model where $\cat$ replaces $\relu$ for the first 7 conv layers and compare the invariance scores with the first 7 conv layers from the baseline $\relu$ model.
Section~\ref{invar_score} in the supplementary materials details how the invariance scores are measured. 
%

%
Figure~\ref{fig:invar_cifar} plots the invariance scores for networks trained on CIFAR-10, CIFAR-100, and ImageNet respectively. 
The invariance scores of $\cat$ models are consistently higher than those of $\relu$ models.
%
For CIFAR-10 and CIFAR-100, there is a big increase between conv2 and conv3 then again between conv4 and conv6, which are due to max-pooling layer extracting shift invariance features. 
%
We also observe that although as a general trend, the invariance scores increase while going deeper into the networks--consistent with the observations from~\citep{goodfellow2009measuring}, the progression is not monotonic. 
%
This interesting observation suggests the potentially diverse functionality of different layers in the CNN, which would be worthwhile for future investigation.

In particular, the scores of ImageNet $\relu$ model attain local maximum at conv1, conv4 and conv7 layers. It inspires us to design the architecture where $\cat$ are placed after conv1, 4, and 7 layers to encourage invariance representations while halving the number of filters to limit model capacity.
%
Interestingly, this architecture achieves the best top1 and top5 recognition results when averaging scores from $10$ patches.
%

\begin{figure*}[t]
\subfigure[Original image]{\includegraphics[width=0.19\textwidth]{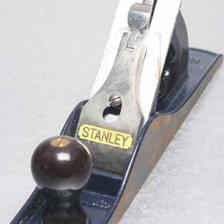}}
\subfigure[conv1]{\includegraphics[width=0.19\textwidth]{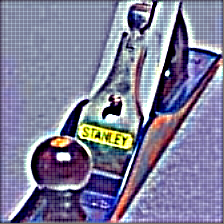}}
\subfigure[conv2]{\includegraphics[width=0.19\textwidth]{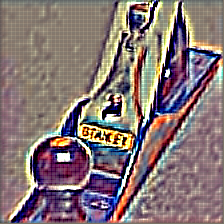}}
\subfigure[conv3]{\includegraphics[width=0.19\textwidth]{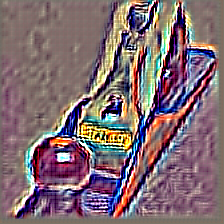}}
\subfigure[conv4]{\includegraphics[width=0.19\textwidth]{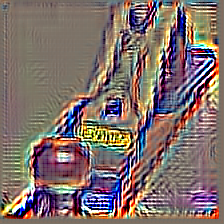}}
\vspace{-0.05in}
\caption{\textbf{$\cat$ Model Reconstructions.} We use a simple linear reconstruction algorithm (see Algorithm~\ref{Recon_1} in the supplementary materials) to reconstruct the original image from conv1-conv4 features (left to right). The image is best viewed in color/screen.
}\label{fig:recon_1_4}
\vspace{-0.1in}
\end{figure*}

\subsection{Revisiting the Reconstruction Property}\label{sec:feature}
%
%
%
In Section~\ref{intuition:conj}, we observe that lower layer convolution filters from $\relu$ models form negatively-correlated pairs.  
Does the pairing phenomenon still exist for $\cat$ models?
%
We take our best $\cat$ model trained on ImageNet (where the first 4 conv layers are integrated with $\cat$) and repeat the histogram experiments to generate Figure~\ref{fig:dist_allconc3}.
%
In clear contrast to Figure~\ref{fig:dist_Alex}, the distributions of  $\nmu^w_i$ from $\cat$ model well align with the distributions of $\nmu^r_i$ from random Gaussian filters. 
%
In other words, each lower layer convolution filter now uniquely spans its own direction without a negatively correlated pairing filter, while $\cat$ implicitly plays the role of ``pair-grouping".

The empirical gap between $\cat$ and $\abs$ justifies that both modulus and phase information are essential in learning deep CNN features. 
In addition, to ensure that the outgoing weights for the positive and negative phase are not merely negations of each other, we measure their correlations for the conv1-7 $\cat$ model trained on ImageNet. 
Table~\ref{tab:correlation} compares the averaged correlation between the (normalized) positive-negative-pair (pair) outgoing weights and the (normalized) unmatched-pair (non-pair) outgoing weights.
The pair correlations are marginally higher than the non-pair ones but both are on average far below $1$ for all layers. 
%
This suggests that, in contrast to $\abs$, the $\cat$ network does not simply focus on the modulus information but imposes different manipulation over the opposite phases.

In Section~\ref{sec:recon}, we mathematically characterize the reconstruction property of convolution layers with $\cat$. 
Proposition~\ref{recon1} claims that the part of an input spanned by the shifts of the filters can be fully recovered. 
ImageNet contains a large number of training images from a wide variety of categories; the convolution filters learned from ImageNet are thus expected to be diverse enough to describe the domain of natural images. 
%
Hence, to qualitatively verify the result from Proposition~\ref{recon1}, we can directly invert features from our best $\cat$ model trained on ImageNet via the simple reconstruction algorithm described in the proof of Proposition~\ref{recon1} (Algorithm~\ref{Recon_1} in the supplementary materials).
Figure~\ref{fig:recon_1_4} shows an image from the validation set along with its reconstructions using conv1-conv4 features (see Section~\ref{image_recon} in the supplementary materials for more reconstruction examples).
%
Unlike other reconstruction methods~\cite{dosovitskiy2015inverting,mahendran2014understanding}, our algorithm does not involve any additional learning.
Nevertheless, it still produces reasonable reconstructions, which supports our theoretical claim in Proposition~\ref{recon1}. 

For the convolution layers involving max-pooling operation, it is less straightforward to perform direct reconstruction. 
Yet we evaluate the conv+$\cat$+max-pooling reconstruction power via measuring properties of the convolution filters and the details are elaborated in Section~\ref{ReconRatio} of the supplementary materials. 

\subsubsection*{Acknowledgments}
We are grateful to Erik Brinkman, Harry Altman and Mark Rudelson for their helpful comments and support. We acknowledge Yuting Zhang and Anna Gilbert for discussions during the preliminary stage of this work. H. Lee was supported in part by ONR N00014-13-1-0762 and NSF CAREER IIS-1453651. We thank Technicolor Research for providing resources and NVIDIA for the donation of GPUs.
\bibliography{example_paper}
\bibliographystyle{icml2016}

\renewcommand{\thesection}{S\arabic{section}}   
\renewcommand{\thetable}{S\arabic{table}}   
\renewcommand{\thefigure}{S\arabic{figure}}
\renewcommand{\theequation}{S\arabic{equation}}

\newpage
\appendix

\clearpage
\part*{Appendix}
\setcounter{equation}{0}
\setcounter{figure}{0}
\setcounter{table}{0}
\setcounter{page}{1}
\makeatletter
\renewcommand{\theequation}{S\arabic{equation}}
\renewcommand{\thefigure}{S\arabic{figure}}
\renewcommand{\bibnumfmt}[1]{[S#1]}
\renewcommand{\citenumfont}[1]{S#1}

\section{Reconstruction Property Proofs}\label{sec:proof_recon}
\subsection{Non-Max-Pooling Case}\label{nonmax}
\begin{proposition}\label{recon1_sup}
Let $x\in\mathbb{R}^{D}$ be an input vector and $W$ be the $D$-by-$K$ matrix whose columns vectors are composed of $w_i \in\mathbb{R}^l, i = 1, \ldots ,K$ convolution filters. Furthermore, let $x = x' +  (x - x')$, where $x' \in\mathrm{range}(W)$ and $x-x' \in \mathrm{ker}(W)$. Then we can reconstruct $x'$ with $\cnn(x)$, where $\cnn(x)\triangleq \cat\left(W^{T}x\right)$.
\end{proposition}
\begin{algorithm}[t]
\caption{Reconstruction over a single convolution region without max-pooling \label{Recon_1}}
\begin{algorithmic}[1]
\STATE $\cnn(x)\gets$ conv features.
\STATE $W\gets$ weight matrix.
\STATE Obtain the linear responses after convolution by reverting $\cat$:  $z = \rho_c^{-1}(\cnn(x))$.
\STATE Compute the Moore Penrose pseudoinverse of $W^T$, $(W^T)^{+}$.
\STATE Obtain the final reconstruction: $x' = (W^T)^{+}z$.
\end{algorithmic}
\end{algorithm}
\begin{algorithm}[t]
\caption{Reconstruction over a single max-pooling region\label{Recon_2}}
\begin{algorithmic}[1]
\STATE $\cnn(x)\gets$ conv features after max-pooling.
\STATE $\widehat{W}_x\gets$ weight matrix consisting of shifted conv filters that are activated by $x$.
\STATE Obtain the linear responses after convolution by reverting $\cat$:  $z = \rho_c^{-1}(\cnn(x))$.
\STATE Compute the Moore Penrose pseudoinverse of $\widehat{W}_x^T$, $(\widehat{W}_x^T)^{+}$.
\STATE Obtain the final reconstruction: $x' = (\widehat{W}_x^T)^{+}z$.
\end{algorithmic}
\end{algorithm}
\begin{proof}
We show $x'$ can be reconstructed from $\cnn(x)$ by providing a simple reconstruction algorithm described by Algorithm~\ref{Recon_1}.
First, apply the inverse function of $\cat$ on $\cnn(x)$: $z = \rho_c^{-1}(\cnn(x))$.
Then, compute the Moore Penrose pseudoinverse of $W^T$, denote by $(W^T)^{+}$. 
By definition $Q = (W^T)^{+} W^T $ is the orthogonal projector onto $\mathrm{range}(W)$, therefore we can obtain $x' = (W^T)^{+}z$.
%
\end{proof}
\subsection{Max-Pooling Case}\label{maxpooling}
\paragraph{Problem Setup.} 
Again, let $x\in\mathbb{R}^{D}$ be an input vector 
 and $w_{i}\in\mathbb{R}^{\ell}$, $i=1,\ldots,K$ be convolution filters.
We denote $w_{i}^{j}\in\mathbb{R}^{D}$ the $j^{\text{th}}$ coordinate shift of the convolution filter $w_{i}$ with a fixed stride length of $s$, i.e., $w_{i}^{j}[(j-1)s+k] = w_{i}[k]$ for $k = 1,\ldots,\ell$, and $0$'s for the rest of entries in the vector.
Here, we assume $D-\ell$ is divisible by $s$ and thus there are $n=\frac{D-\ell}{s}+1$ shifts for each $w_i$.
%
%
We define $W$ to be the $D \times nK$ matrix whose columns are the shifts $w_i^j$, $j=1,\ldots,n$, for $w_{i}$; the columns of $W$ are divided into $K$ blocks, with each block consisting of $n$ shifts of a single filter.
%
The conv + $\cat$ + max-pooling layer can be defined by first multiplying an input signal $x$ by the matrix $W^{T}$ (conv), separating positive and negative phases then applying the $\relu$ non-linearity ($\cat$), and selecting the maximum value in each of the $K$ block (max-pooling).
The operation is denoted as $\cnn:\mathbb{R}^{D}\rightarrow\mathbb{R}^{2K}$ such that $\cnn(x)\triangleq g\left(W^{T}x\right)$, where $g\triangleq \pool\circ\cat$.
Figure~\ref{fig:simple_conv_pool_illustration} illustrates an example of the problem setting. 
\begin{figure}[t]
\centering
\includegraphics[width=0.45\textwidth]{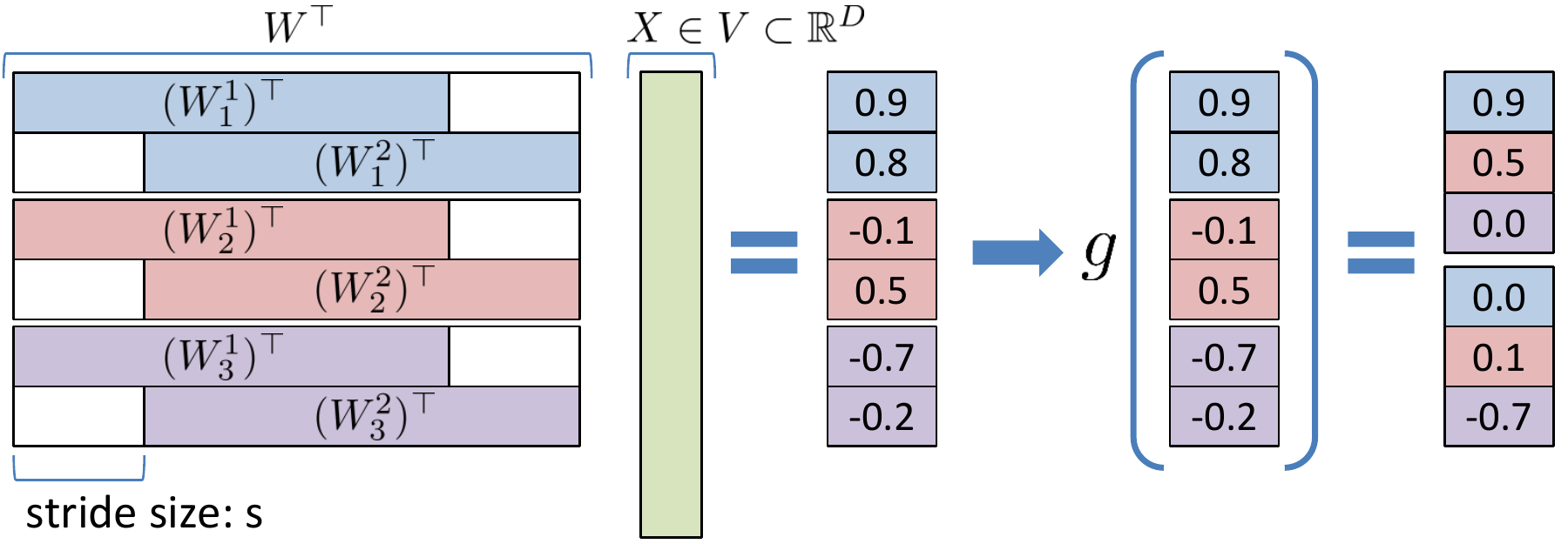}
\caption{\textbf{An illustration of convolution, $\cat$, and max-pooling operation.} For simplicity, we describe with 3 convolution filters ($W_{1},W_{2},W_{3}$) with stride of $s$, and with $2\times 2$ pooling. In Figure (a), $g$ denotes $\cat$ followed by the max-pooling operation.\label{fig:simple_conv_pool_illustration}}
\end{figure}
\paragraph{Assumption.}To reach a non-trivial bound when max-pooling is present, we put a constraint on the input space $\mathcal{V}$: $\forall x \in \mathcal{V}$, there exists $\{c_{i}^{j}\}_{i=1,\cdots,K}^{j=1,\cdots,n}$ such that
\begin{equation}\label{assumption1}
x = \sum_{i=1}^{K}\sum_{j=1}^{n} c_{i}^{j} w_{i}^{j}, \text{ where } \sum_{j=1}^{n} \mathbf{1}\{c_{i}^{j} > 0\} \leq 1, \;\;\forall i.
\end{equation}
In other words, we assume that an input $x$ is a linear combination of the shifted convolution filters $\{w_{i}^{j}\}_{i=1,\cdots,K}^{j=1,\cdots,n}$ such that over a single max-pooling region, only one of the shifts participates: $ \sum_{j=1}^{n} \mathbf{1}\{c_{i}^{j} > 0\} \leq 1$:
a slight translation of an object or viewpoint change does not alter the nature of a natural image, which is how max-pooling generates shift invariant features by taking away some fine-scaled locality information.
%
%

Next, we denote the matrix consisting of the shifts whose corresponding $c_i^j$'s are non-zero by $W_x$ , and the vector consisting of the non-zero $c_i^j$'s by $\vct{c}_x$, i.e. $W_x\vct{c}_x = x$. 
%
%
%
Also, we denote the matrix consisting of the shifts whose activation is positive and selected after max-pooling operation by $\widehat{W}^{+}_x$, negative by $\widehat{W}^{-}_x$.  
Let $\widehat{W}_x \triangleq \left[\widehat{W}^{+}_x, \widehat{W}^{-}_x\right]$.
Finally, we give notation, $\widetilde{W}_x$, to the matrix consisting of a subset of $\widehat{W}_x$, such that the $i$th column comes from $\widehat{W}^{+}_x$ if $c_i^j\geq 0$ or from $\widehat{W}^{-}_x$ if otherwise.
%

\paragraph{Frame Theory.} 
Before proceeding to the main theorem and its proof, we would like to introduce more tools from Frame Theory.
\begin{Definition}\label{frame_def}
A frame is a set of elements of a vector space $V$, $\{\f_k\}_{k=1, \cdots , K}$, which satisfies the frame condition: there exist two real numbers $C_1$ and $C_2$, the \emph{frame bounds}, such that $0<C_1\leq C_2<\infty$, and $\forall v\in V$
\begin{equation*}
	C_1\|v\|_2^2 \leq \sum^K_{k=1} |\langle v, \f_i\rangle|^2 \leq C_2 \|v\|_2^2.
\end{equation*} \cite{christensen2003introduction}
\end{Definition}
\begin{proposition}\label{frame}
Let $\{\f_k\}_{k=1,\ldots,K}$ be a sequence in $V$, then $\{\f_k\}$ is a frame for $\mathrm{span}\{\f_k\}$. Hence, $\{\f_k\}$ is a frame for $V$ if and only if $V = \mathrm{span}\{\f_k\}$\footnote{There exist infinite
spanning sets that are not frames, but we will not be concerned with
those here since we only deal with finite dimensional vector spaces.}.  \cite{christensen2003introduction}
\end{proposition}
\begin{Definition}
Consider now $V$ equipped with a frame $\{\f_k\}_{k=1,\ldots , K}$. The \emph{Analysis Operator}, $\T: V\to\R^K,$ is defined by $\T v = \{\langle v, \f_k\rangle\}_{k=1, \ldots, K}$. The \emph{Synthesis Operator}, $\T^*: \R^K\to V$, is defined by  $\T^*\{c_k\}_{k=1,\ldots , K} = \sum_{k=1}^K c_k \f_k$, which is the adjoint of the Analysis Operator.
The \emph{Frame Operator}, $\SO: V\to V$, is defined to be the composition of $\T$ with its adjoint: 
\[
\SO v = \T^*\T v.
\]
The Frame Operator is always invertible. \cite{christensen2003introduction}
\end{Definition}
\begin{theorem}\label{frame_bound}
The optimal lower frame bound $C_1$ is the smallest eigenvalue of $\SO$; the optimal upper frame bound $C_2$ is the largest eigenvalue of $\SO$.  \cite{christensen2003introduction}
\end{theorem}
We would also like to investigate the matrix representation of the operators $\T, \T^*$ and $\SO$. 
Consider $V$, a subspace of ${\mathbb R}^D$, equipped with a frame $\{\f_k\}_{k=1,\cdots , K}$.
Let $U\in R^{D \times d}$ be a matrix whose column vectors form an orthonormal
basis for $V$ (here $d$ is the
dimension of $V$). 
Choosing $U$ as the basis for $V$ and choosing the standard basis $\{e_k\}_{k=1,\cdots , K}$ as the basis for $R^K$, the matrix representation of $\T$ is $\Tm= W^TU$, where $W$ is the matrix whose column vectors are $\{\f^T_k\}_{k=1,\cdots , K}$. 
Its transpose, $\Ttm$, is the matrix representation for $\T^*$; the matrix representation for $\SO$ is $\Sm= \tilde{\T}^*\tilde{\T}$. 
\begin{lemma}\label{normlowerbound}
Let $x\in \R^D$ and $W$ an $D$-by-$K$ matrix. 
If $x\in \mathrm{range}(W)$, then $\sigma_{\min}\|x\|_2 \leq \|W^Tx\|_2
\leq \sigma_{\max}\|x\|_2$, where $\sigma_{\min}$ and $\sigma_{\max}$ are the least and largest singular value of $V$ respectively. 
\end{lemma}
\begin{proof}
By Proposition~\ref{frame}, the columns in $W$ form a frame for $\mathrm{range}(W)$.
Let $U$ be an orthonormal basis for $\mathrm{range}(W)$. Then the matrix representation under $U$ for the Analysis Operator, $\T$, is $\Tm = W^TU$, and the corresponding representation for $x$ under $U$ is $U^T x$. Now, by Theorem~\ref{frame_bound}, we have: 
\begin{equation*}
  \lambda_{\min}\|x\|^2_2\leq \| \Tm x\|^2_2 = \|W^T U U^T x\|^2_2 = \|W^T x\|^2_2,
\end{equation*}
where $\lambda_1$ is the least eigenvalue of $\Sm$. 
Therefore, we have $\sigma_{\min}\|x\|_2 \leq \|W^T x\|_2$, where $\sigma_{\min}$ is the least singular value of $W$. 
Lastly, by the definition of operator-induced matrix norm, we have the upper bound  $\|W^Tx\|_2 \leq \sigma_{\max}\|x\|_2$
\end{proof}

\paragraph{Reconstruction Property.}
Now we are ready to present the theorem that characterizes  the reconstruction property of the conv+$\cat$+max-pooling operation.
\begin{theorem}\label{recon_error}
Let $x\in \mathcal{V}$ and satisfy the assumption from Equation~\eqref{assumption1}. 
Then we can obtain $x'$, the reconstruction of $x$ using $\cnn(x)$ 
such that
\begin{equation*}\label{recon_ratio}
\frac{\|x-x'\|_2}{\|x\|_2}\leq \sqrt{\frac{\widetilde{\lambda}_{\max}-\lambda_{\min}  }{\lambda_{\min} }},
\end{equation*}
where $\lambda_{\min}$ and $\widetilde{\lambda}_{\max}$ are square of the minimum and maximum singular values of $W_x$ and $\widetilde{W}_x$ respectively.
\end{theorem}
\begin{proof}
We use similar method to reconstruct as described by Algorithm~\ref{Recon_2}: 
first reverse the $\cat$ activation and obtain $z = \rho_c^{-1}(\cnn(x))$;
then compute the Moore Penrose pseudoinverse of $\widehat{W}_x^T$, denote by $(\widehat{W}_x^T)^{+}$; 
finally, obtain $x' = (\widehat{W}_x^T)^{+}z$, since by definition, $Q = (\widehat{W}_x^T)^{+} \widehat{W}_x^T $ is the orthogonal projector onto $\mathrm{range}(\widehat{W}_x)$. 
To proceed the proof, we denote the subset of $z$ which matches the corresponding activation of the filters from $\widetilde{W}_x$ by $\tilde{z}$, compute the Morre Penrose pseudoinverse of $\widetilde{W}_x$ and obtain $\tilde{x} = (\widetilde{W}_x^T)^{+}\tilde{z}$. 
Note that since $\mathrm{range}(\widetilde{W}_x)$ is a subspace of $\mathrm{range}(\widehat{W}_x)$, therefore, the reconstruction $x'$ will always be equal or better than $\tilde{x}$, i.e. $\|x-x'\|_2\leq\|x-\tilde{x}\|_2$. 
From Lemma~\ref{normlowerbound}, the nature of max-pooling and the assumption on $x$ (Equation~\ref{assumption1}), we derive the following inequality
\begin{align*}
\lambda_{\min} \|x\|^2_2 \leq\|W^T_x x\|_2 &\leq \|\widetilde{W}^T_x x\|^2_2 \\
& = \|\widetilde{W}^T_x \tilde{x}\|^2_2 \leq \tilde{\lambda}_{\max}\|\tilde{x}\|^2_2,
\end{align*}
where $\lambda_{\min}$ and $\tilde{\lambda}_{\max}$ are square of the minimum and maximum singular values of $W_x$ and $\widetilde{W}_x$ respectively. 

Because $\tilde{x}$ is the orthogonal projection of $x$ on to $\mathrm{range}(\widetilde{W}_x)$, thus $\|x\|^2_2 = \|\tilde{x}\|^2_2 + \|x-\tilde{x}\|^2_2$. Now substitute $\|x\|_2^2$ with $\|\tilde{x}\|^2_2 + \|x-\tilde{x}\|^2_2$, we have: 
$$
\begin{aligned}
&\lambda_{\min}  (\|\tilde{x}\|^2_2 + \|x-\tilde{x}\|^2_2) \leq \tilde{\lambda}_{\max}\|\tilde{x}\|^2_2\\
&\|x-\tilde{x}\|^2_2\leq \frac{\tilde{\lambda}_{\max}-\lambda_{\min}  }{\tilde{\lambda}_{\min} }\|\tilde{x}\|^2_2\\
&\|x-x'\|^2_2\leq \frac{\tilde{\lambda}_{\max}-\lambda_{\min}  }{\lambda_{\min} }\|x\|^2_2\\
&\|x-x'\|_2\leq \sqrt{\frac{\tilde{\lambda}_{\max}-\lambda_{\min}  }{\lambda_{\min} }}\|x\|_2\\
&\frac{\|x-x'\|_2}{\|x\|_2}\leq \sqrt{\frac{\widetilde{\lambda}_{\max}-\lambda_{\min}  }{\lambda_{\min} }}.
\end{aligned}
$$
\end{proof}
We refer to the term $\frac{\|x-x'\|_2}{\|x\|_2}$ as the \emph{reconstruction ratio} in later discussions.
\section{Proof of Model Complexity Bound}
\label{sec:proof_complex}
\begin{Definition}
\emph{(Rademacher Complexity)} For a sample $S = \{x_1,\cdots , x_L\}$ generated by a distribution $D$ on set $X$ and a real-valued function class $\F$ in domain $X$, the empirical Rademacher complexity of $\F$ is the random variable: 
\[
\eR(\F) = \Expect_\sigma\left[\sum_{f\in\F}|\frac{2}{L} \sigma_i f(x_i)| \bigg| x_1,\cdots, x_L \right],
\]
where $\sigma_i$'s are independent uniform $\{\pm 1 \}$-valued (Rademacher) random variables. The Rademacher complexity of $\F$ is $R_L(\F) = \Expect_S \left[ \eR(\F)\right]$
\end{Definition}
\begin{lemma}\label{comp}
\emph{(Composition Lemma)} Assume $\rho:\R\to\R$ is a $L_\rho$-Lipschitz continuous function, i.e. , $|\rho(x)-\rho(y)|\leq L_\rho |x-y|$. Then $\eR(\rho\circ\F) = L_\rho \eR(\F)$. 
\end{lemma}
\begin{proposition}\label{networkbound}
\emph{(Network Layer Bound)} Let $\mathcal{G}$ be the class of real functions $\R^{\din}\to\R$ with input dimension $\F$, that is, $\mathcal{G} = [\F]_{j=1}^{\din}$ and $\mathcal{H}$ is a linear transform function parametrized by $W$ with $\|W\|_2\leq B$, then $\eR(\mathcal{H}\circ\mathcal{G})\leq\sqrt{\din} B \eR(\F)$. \cite{wan2013regularization}
\end{proposition}
\begin{corollary}\label{reluRC}
By Lemma~\ref{comp}, Proposition~\ref{networkbound}, and the fact that $\relu$ is $1$-Lipschitz, we know that $\eR(\relu \circ\mathcal{G}) =  \eR(\mathcal{G})$ and that $\eR(\mathcal{H}\circ \relu \circ\mathcal{G})\leq\sqrt{\din} B \eR(\F)$.
\end{corollary}
\begin{theorem}
\ref{RCBound}
Let $\mathcal{G}$ be the class of real functions $\R^{\din}\to \R$ with input dimension $\F$, that is, $\mathcal{G} = [\F]_{j=1}^{\din}$. Let $\mathcal{H}$ be a linear transform function from $\R^{2\din}$ to $\R$, parametrized by $W$, where $\|W\|_2\leq B$. Then $\eR (\mathcal{H}\circ \rho_c \circ \mathcal{G}) \leq \sqrt{\din} B \eR (\F)$.
\end{theorem} 
Recall from Definition~\ref{crelu}, $\rho_c$ is the $\cat$ formulation. 
\begin{proof}
\begin{align}
&\eR(\mathcal{H}\circ \rho_c \circ \mathcal{G}) = \Expect_\sigma \left[\sup_{h\in\mathcal{H}, g\in\mathcal{G}}|\frac{2}{L}\sum_{i=1}^L \sigma_i h\circ \rho_c \circ g(x_i)|\right]\\
&=\Expect_\sigma \left[ \sup_{\|W\|\leq B, g\in\mathcal{G}}|\langle W, \frac{2}{L}\sum_{i=1}^{L}\sigma_i  \rho_c\circ g(x_i) \rangle |\right]\\
& \leq B \Expect_\sigma \left[ \sup_{ f\in\F}\|\left[ \frac{2}{L}\sum_{i=1}^{L}\sigma^j_i  \rho_c\circ f^j(x_i)\right]_{j=1}^{\din} \|_2\right]\\
& = B \Expect_\sigma \left[ \sup_{ f\in\F}\|\left[ \frac{2}{L}\sum_{i=1}^{L}\sigma^j_i  f^j(x_i)\right]_{j=1}^{\din} \|_2\right]\\
&= B\sqrt{\din} \Expect_\sigma \left[ \sup_{f\in\F} |\frac{2}{L}\sum_{i=1}^L \sigma_i f(x_i)|\right]\\
& = \sqrt{\din} B \eR(\F). 
\end{align}
From (S1) to (S2), use the definition of linear transformation and inner product. 
From (S2) to (S3), use Cauchy-Schwarz inequality and the assumption that $\|W\|_2\leq B$.
From (S3) to (S4), use the definition of $\cat$ and $l^2$ norm.
From (S4) to (S5), use the definition of $l^2$ norm and $\sup$ operator.
From (S5) to (S6), use the definition of $\eR$
\end{proof}
We see that $\cat$ followed by linear transformation reaches the same Rademacher complexity bound as $\relu$ followed by linear transformation with the same input dimension.
\section{Reconstruction Ratio}\label{ReconRatio}
\begin{table}[t]
\caption{\textbf{Empirical mean of the reconstruction ratios.}
Reconstruct the sampled images from test set using the features after $\cat$ and max-pooling; then calculate the reconstruction ratio, $\|x-x'\|_2/\|x\|_2$. \label{tab:reconstructionration}}
\vspace{0.05in}
\centering
\small
\begin{tabular}{c|c|c} 
\hline
\multicolumn{3}{c}{CIFAR-10}\\ \hline
layer & learned & random  \\ \hline
conv2 & 0.92 {\scriptsize$\pm 0.0002$}  & 0.99 {\scriptsize$\pm 0.00005$}   \\ \hline
conv5& 0.96 {\scriptsize$\pm 0.0003$} & 0.99 {\scriptsize$\pm 0.00005$}  \\ \hline 
\hline
\multicolumn{3}{c}{CIFAR-100}\\ \hline
layer & learned & random   \\ \hline
conv2 & 0.93 {\scriptsize$\pm 0.0002$}& 0.99 {\scriptsize$\pm 0.00005$}   \\ \hline
conv5& 0.96 {\scriptsize$\pm 0.0001$}& 0.99 {\scriptsize$\pm 0.00005$}  \\ \hline 
\end{tabular}
\vspace{-0.15in}
\end{table}
Recall that Theorem~\ref{recon_error} characterizes the reconstruction property when max-pooling is added after $\cat$.
%
As an example, we study the all-conv $\cat$ (half) models used for CIFAR-10/100 experiments.
%
In this model, conv2 and conv5 layers are followed by max-pooling.
CIFAR images are much less diverse than those from ImageNet.
Instead of directly inverting features all the way back to the original images, we empirically calculate the reconstruction ratio, $\|x-x'\|_2/\|x\|_2$.
We sample testing examples, extract pooled features after conv2(conv5) layer and reconstruct features from the previous layer via Algorithm~\ref{Recon_2}. 
%
%
%
To compare, we perform the same procedures on random convolution filters\footnote{Each entry is sampled from standard normal distribution.}.
%
Essentially, convolution imposes structured zeros to the random $\widetilde{W}_x$; there has not been published results on random subspace projection with such structured zeros. 
In a simplified setting without structured zeros, 
i.e. no convolution, it is straightforward to show that the expected reconstruction ratio is $\sqrt{\frac{D-K}{D}}$ (Theorem~\ref{random_subspace}), where, in our case, $D=48(96)\times 5\times 5$ and $K = 48(96)$ for conv2(conv5) layer.
%
%
%
Table~\ref{tab:reconstructionration} compares between the empirical mean of reconstruction ratios using learned filters and random filters:
random filters only recover $1\%$ of the original input, whereas the learned filters span more of the input domain. 
\begin{theorem}\label{random_subspace} Let $x\in \R^{D}$, and let $x_s\in \R^{D}$ be its projection onto a random subspace of dimension $D_2$, then
\[
\Expect \left [\frac{\|x_s\|_2}{\|x\|_2}\right ] = \sqrt{\frac{D_s}{D}}
\]
\end{theorem}
\begin{proof}
Without loss of generality, let $\|x\|_2=1$.
Projecting a fixed $x$ onto a random subspace of dimension $D_s$ is equivalent of projecting a random unit-norm vector $z = (z_1,z_2, \cdots , z_D)^T$ onto a fixed subspace of dimension $D_s$ thanks to the rotational invariance of inner product. 
Without loss of generality, assume the fixed subspace here is spanned by the first $D_s$ standard basis covering the first $D_2$ coordinates of $z$. 
Then the resulting projection is $z_s = (z_1,z_2, \cdots , z_{D_s}, 0, \cdots, 0)$.

Because $z$ is unit norm, we have
\[
\Expect \left[ \|z\|^2_2 \right] = \Expect \left[ \sum_{i=1}^D z^2_i \right] = 1.
\]
Because each entry of $z$, $z_i$, is identically distributed, we have
\[
\Expect \left[\|z_s\|^2_2 \right] =  \Expect \left[ \sum_{i=1}^{D_s} z^2_i \right] = \frac{D_s}{D}.
\]
Together we have
\[
\Expect \left [\frac{\|x_s\|_2}{\|x\|_2}\right ] = \Expect \left [\frac{\|z_s\|_2}{\|z\|_2}\right ] =\sqrt{\frac{D_s}{D}}.
\]
\end{proof}
\section{Invariance Score}\label{invar_score}
We use consistent terminology employed by \citet{goodfellow2009measuring} to illustrate the calculation of the invariance scores.

For CIFAR-10/100, we utilize all 50k testing images to calculate the invariance scores; for ImageNet, we take the center crop from 5k randomly sampled validation images

For each individual filter, we calculate its own \emph{firing} threshold, such that it is fired one percent of the time, i.e. the \emph{global firing rate} is $0.01$. 
For $\relu$ models, we zero out all the negative negative responses when calculating the threshold; for $\cat$ models, we take the absolute value. 

To build the set of semantically similar stimuli for each testing image $x$, we apply horizontal flip, 15 degree rotation and translation.
For CIFAR-10/100, translation is composed of horizontal/vertical shifts by 3 pixels; for ImageNet, translation is composed of cropping from the 4 corners. 

Because our setup is convolutional, we consider a filter to be fired only if both the transformed stimulus and the original testing example fire the same convolution filter at the \emph{same} spatial location.

At the end, for each convolution layer, we average the invariance scores of all the filters at this layer to form the final score. 

\section{Implementation Details on ImageNet Models}
\label{sec:detail-training}
The networks from Table~\ref{allconv},~\ref{allconc1},~\ref{allconc2},and~\ref{allconc3}, where the number of convolution filters after $\cat$ are kept the same, are optimized using SGD with mini-batch size of $64$ examples and fixed momentum $0.9$. The learning rate and weight decay is adapted using the following schedule: epoch 1-10, $1\mathrm{e}{-2}$ and $5\mathrm{e}{-4}$; epoch 11-20, $1\mathrm{e}{-3}$ and $5\mathrm{e}{-4}$; epoch 21-25, $1\mathrm{e}{-4}$ and $5\mathrm{e}{-4}$; epoch 26-30, $5\mathrm{e}{-5}$ and $0$; epoch 31-35, $1\mathrm{e}{-5}$ and $0$; epoch 36-40, $5 \mathrm{e}{-6}$ and $0$; epoch 41-45, $1 \mathrm{e}{-6}$ and $0$. 

The networks from Table~\ref{allconch} and~\ref{allconchonly}, where the number of convolution filters after $\cat$ are reduced by half, are optimized using Adam with an initial learning rate $0.0002$ and mini-batch size of $64$ examples for $100$ epochs.

\onecolumn
\section{Details of Network Architecture}
\label{sec:model}

\begin{table*}[htbp]
\caption{(Left) Baseline or $\abs$ and (right) baseline (double) models used for CIFAR-10/100 experiment. ``avg'' refers average pooling.}\label{conpo}
\vspace{0.05in}\centering
{
\begin{tabular}{c|c|c!{\vrule width 1pt}c|c}
\hline
 & \multicolumn{2}{c!{\vrule width 1pt}}{Baseline/$\abs$} & \multicolumn{2}{c}{Baseline (double)}\\ \hline
Layer& kernel, stride, padding & activation & kernel, stride, padding & activation\\
\hline
conv1 & 3$\times$3$\times$3$\times$96, 1, 1 & $\relu$/$\abs$ & 3$\times$3$\times$3$\times$192, 1, 1 & $\relu$\\
\hline
conv2 & 3$\times$3$\times$96$\times$96, 1, 1 & $\relu$/$\abs$ & 3$\times$3$\times$192$\times$192, 1, 1 & $\relu$\\
\hline
pool1 & 3$\times$3, 2, 0 & max & 3$\times$3, 2, 0 & max\\
\hline
conv3 & 3$\times$3$\times$96$\times$192, 1, 1 & $\relu$/$\abs$ & 3$\times$3$\times$192$\times$384, 1, 1 & $\relu$\\
\hline
conv4 & 3$\times$3$\times$192$\times$192, 1, 1 & $\relu$/$\abs$ & 3$\times$3$\times$384$\times$384, 1, 1 & $\relu$\\
\hline
conv5 & 3$\times$3$\times$192$\times$192, 1, 1 & $\relu$/$\abs$ & 3$\times$3$\times$384$\times$384, 1, 1 & $\relu$\\
\hline
pool2 & 3$\times$3, 2, 0 & max & 3$\times$3, 2, 0 & max\\
\hline
conv6 & 3$\times$3$\times$192$\times$192, 1, 1 & $\relu$/$\abs$ & 3$\times$3$\times$384$\times$384, 1, 1 & $\relu$\\
\hline
conv7 & 1$\times$1$\times$192$\times$192, 1, 1 & $\relu$/$\abs$ & 1$\times$1$\times$384$\times$384, 1, 1 & $\relu$\\
\hline
conv8 & 1$\times$1$\times$192$\times$10/100, 1, 0 & $\relu$/$\abs$ & 1$\times$1$\times$384$\times$10/100, 1, 0 & $\relu$\\
\hline
pool3 & 10$\times$10 (100 for CIFAR-100) & avg & 10$\times$10 (100 for CIFAR-100) & avg\\
\hline
\hline
\end{tabular}
}
\end{table*}

\begin{table*}[htbp]
\caption{(Left) $\cat$ and (right) $\cat$ (half) models used for CIFAR-10/100 experiment.}\label{copoc}
\vspace{0.05in}\centering
{
\begin{tabular}{c|c|c!{\vrule width 1pt}c|c}
\hline
 & \multicolumn{2}{c!{\vrule width 1pt}}{$\cat$} & \multicolumn{2}{c}{$\cat$ (half)}\\ \hline
Layer& kernel, stride, padding & activation & kernel, stride, padding & activation\\
\hline
conv1 & 3$\times$3$\times$3$\times$96, 1, 1 & $\cat$ & 3$\times$3$\times$3$\times$48, 1, 1 & $\cat$\\
\hline
conv2 & 3$\times$3$\times$192$\times$96, 1, 1 & $\cat$ & 3$\times$3$\times$96$\times$48, 1, 1 & $\cat$\\
\hline
pool1 & 3$\times$3, 2, 0 & max & 3$\times$3, 2, 0 & max\\
\hline
conv3 & 3$\times$3$\times$192$\times$192, 1, 1 & $\cat$ & 3$\times$3$\times$96$\times$48, 1, 1 & $\cat$\\
\hline
conv4 & 3$\times$3$\times$384$\times$192, 1, 1 & $\cat$ & 3$\times$3$\times$96$\times$96, 1, 1 & $\cat$\\
\hline
conv5 & 3$\times$3$\times$384$\times$192, 1, 1 & $\cat$ & 3$\times$3$\times$192$\times$96, 1, 1 & $\cat$\\
\hline
pool2 & 3$\times$3, 2, 0 & max & 3$\times$3, 2, 0 & max\\
\hline
conv6 & 3$\times$3$\times$384$\times$192, 1, 1 & $\cat$ & 3$\times$3$\times$192$\times$96, 1, 1 & $\cat$\\
\hline
conv7 & 1$\times$1$\times$384$\times$192, 1, 1 & $\cat$ & 1$\times$1$\times$192$\times$96, 1, 1 & $\cat$\\
\hline
conv8 & 1$\times$1$\times$384$\times$10/100, 1, 0 & $\relu$ & 1$\times$1$\times$192$\times$10/100, 1, 0 & $\relu$\\
\hline
pool3 & 10$\times$10 (100 for CIFAR-100) & avg & 10$\times$10 (100 for CIFAR-100) & avg\\
\hline
\hline
\end{tabular}
}
\end{table*}

\twocolumn

\newpage
\begin{table}[t]
\caption{VGG for CIFAR-10/100}\label{vgg}
\vspace{0.05in}\centering
{
\begin{tabular}{c|c|c}
\hline
Layer& kernel, stride, padding & activation\\
\hline
conv1 & 3$\times$3$\times$3$\times$64, 1, 1 & BN+$\relu$ \\
\hline
   & dropout with ratio $0.3$ & \\
\hline
conv2 & 3$\times$3$\times$64$\times$64, 1, 1 & BN+$\relu$\\
\hline
pool1 & 2$\times$2, 2, 0 & \\
\hline
conv3 & 3$\times$3$\times$64$\times$128, 1, 1 & BN+$\relu$ \\
\hline
    & dropout with ratio $0.4$ & \\
\hline
conv4 & 3$\times$3$\times$128$\times$128, 1, 1 & BN+$\relu$\\
\hline
pool2 & 2$\times$2, 2, 0 & \\
\hline
conv5 & 3$\times$3$\times$128$\times$256, 1, 1 & BN+$\relu$ \\
\hline
    & dropout with ratio $0.4$ & \\
\hline
conv6 & 3$\times$3$\times$256$\times$256, 1, 1 & BN+$\relu$\\
\hline
    & dropout with ratio $0.4$ & \\
\hline
conv7 & 3$\times$3$\times$256$\times$256, 1, 1 & BN+$\relu$\\
\hline
pool3 & 2$\times$2, 2, 0 & \\
\hline
conv8 & 3$\times$3$\times$256$\times$512, 1, 1 & BN+$\relu$ \\
\hline
    & dropout with ratio $0.4$ & \\
\hline
conv9 & 3$\times$3$\times$512$\times$512, 1, 1 & BN+$\relu$\\
\hline
    & dropout with ratio $0.4$ & \\
\hline
conv10 & 3$\times$3$\times$512$\times$512, 1, 1 & BN+$\relu$\\
\hline
pool4 & 2$\times$2, 2, 0 & \\
\hline
conv11 & 3$\times$3$\times$512$\times$512, 1, 1 & BN+$\relu$ \\
\hline
    & dropout with ratio $0.4$ & \\
\hline
conv12 & 3$\times$3$\times$512$\times$512, 1, 1 & BN+$\relu$\\
\hline
    & dropout with ratio $0.4$ & \\
\hline
conv13 & 3$\times$3$\times$512$\times$512, 1, 1 & BN+$\relu$\\
\hline
pool5 & 2$\times$2, 2, 0 & \\
\hline
    & dropout with ratio $0.5$ & \\
\hline
fc14 & 512$\times$512 & BN+$\relu$\\
\hline
    & dropout with ratio $0.5$ & \\
\hline
fc15 & 512$\times$10/100 & \\
\hline
\hline
\end{tabular}
}
\end{table}

\begin{table}[htbp]
\caption{VGG + (conv1) for CIFAR-10/100}\label{vgg1}
\vspace{0.05in}\centering
{
\begin{tabular}{c|c|c}
\hline
Layer& kernel, stride, padding & activation\\
\hline
conv1 & 3$\times$3$\times$3$\times$32, 1, 1 & $\cat$ \\
\hline
   & dropout with ratio $0.1$ & \\
\hline
conv2& $\cdots$ & \\
\hline
\hline
\end{tabular}
}\vspace{0.2in}
\caption{VGG + (conv1, 3) for CIFAR-10/100}\label{vgg3}
\vspace{0.05in}\centering
{
\begin{tabular}{c|c|c}
\hline
Layer& kernel, stride, padding & activation\\
\hline
conv1 & 3$\times$3$\times$3$\times$32, 1, 1 & $\cat$ \\
\hline
   & dropout with ratio $0.1$ & \\
\hline
conv2 & 3$\times$3$\times$64$\times$64, 1, 1 & BN+$\relu$\\
\hline
pool1 & 2$\times$2, 2, 0 & \\
\hline
conv3 & 3$\times$3$\times$64$\times$64, 1, 1 & $\cat$ \\
\hline
    & dropout with ratio $0.2$ & \\
\hline
conv4 & $\cdots$ & \\
\hline
\hline
\end{tabular}
}\vspace{0.2in}
\caption{VGG + (conv1, 3, 5) for CIFAR-10/100}\label{vgg5}
\vspace{0.05in}\centering
{
\begin{tabular}{c|c|c}
\hline
Layer& kernel, stride, padding & activation\\
\hline
conv1 & 3$\times$3$\times$3$\times$32, 1, 1 & $\cat$ \\
\hline
   & dropout with ratio $0.1$ & \\
\hline
conv2 & 3$\times$3$\times$64$\times$64, 1, 1 & BN+$\relu$\\
\hline
pool1 & 2$\times$2, 2, 0 & \\
\hline
conv3 & 3$\times$3$\times$64$\times$64, 1, 1 & $\cat$ \\
\hline
    & dropout with ratio $0.2$ & \\
\hline
conv4 & 3$\times$3$\times$128$\times$128, 1, 1 & BN+$\relu$\\
\hline
pool2 & 2$\times$2, 2, 0 & \\
\hline
conv5 & 3$\times$3$\times$128$\times$128, 1, 1 & $\cat$ \\
\hline
    & dropout with ratio $0.2$ & \\
\hline
conv6 & 3$\times$3$\times$256$\times$256, 1, 1 & BN+$\relu$\\
\hline
    & dropout with ratio $0.2$ & \\
\hline
conv7 & $\cdots$ & \\
\hline
\hline
\end{tabular}
}
\end{table}

\begin{table}[htbp]
\caption{Baseline for ImageNet}\label{allconv}
\vspace{0.05in}\centering
{
\begin{tabular}{c|c|c}
\hline
Layer & kernel, stride, padding & activation\\
\hline
conv1 & 11$\times$11$\times$3$\times$96, 4,0 & $\relu$\\
\hline
conv2 & 1$\times$1$\times$96$\times$96, 1,0 & $\relu$\\
\hline
conv3 & 3$\times$3$\times$96$\times$96, 2,0 & $\relu$\\
\hline
conv4 & 5$\times$5$\times$96$\times$256, 1, 2 & $\relu$\\
\hline
conv5 & 1$\times$1$\times$256$\times$256, 1,0 & $\relu$\\
\hline
conv6 & 3$\times$3$\times$256$\times$256, 2,0 & $\relu$\\
\hline
conv7 & 3$\times$3$\times$256$\times$384, 1 ,1& $\relu$\\
\hline
conv8 & 1$\times$1$\times$384$\times$384, 1,0 & $\relu$\\
\hline
conv9 & 3$\times$3$\times$384$\times$384, 2,1 & $\relu$\\
\hline
    &  no dropout & \\
\hline
conv10 & 3$\times$3$\times$384$\times$1024, 1,1 & $\relu$\\
\hline
conv11 & 1$\times$1$\times$1024$\times$1024, 1,0 & $\relu$\\
\hline
conv12 & 1$\times$1$\times$1024$\times$1000, 1 & $\relu$\\
\hline
pool & 6$\times$6 average-pooling & \\
\hline
\hline
\end{tabular}
}\vspace{0.1in}
\caption{$\cat$/$\abs$ (conv1-4) for ImageNet}\label{allconc1}
\vspace{0.05in}\centering
{
\begin{tabular}{c|c|c}
\hline
Layer & kernel, stride, padding & activation\\
\hline
conv1 & 11$\times$11$\times$3$\times$96, 4,0 & $\cat$/$\abs$\\
\hline
conv2 & 1$\times$1$\times$192/96$\times$96, 1,0 & $\cat$/$\abs$\\
\hline
conv3 & 3$\times$3$\times$192/96$\times$96, 2,0 & $\cat$/$\abs$\\
\hline
conv4 & 5$\times$5$\times$192/96$\times$256, 1, 2 & $\cat$/$\abs$\\
\hline
conv5 & 1$\times$1$\times$512/256$\times$256, 1,0 & $\relu$\\
\hline
conv6 & 3$\times$3$\times$256$\times$256, 2,0 & $\relu$\\
\hline
conv7 & 3$\times$3$\times$256$\times$384, 1 ,1& $\relu$\\
\hline
conv8 & 1$\times$1$\times$384$\times$384, 1,0 & $\relu$\\
\hline
conv9 & 3$\times$3$\times$384$\times$384, 2,1 & $\relu$\\
\hline
    & no dropout & \\
\hline
conv10 & 3$\times$3$\times$384$\times$1024, 1,1 & $\relu$\\
\hline
conv11 & 1$\times$1$\times$1024$\times$1024, 1,0 & $\relu$\\
\hline
conv12 & 1$\times$1$\times$1024$\times$1000, 1 & $\relu$\\
\hline
pool & 6$\times$6 average-pooling & \\
\hline
\hline
\end{tabular}
}\vspace{0.1in}
\caption{$\cat$/$\abs$ (conv1-7) for ImageNet}\label{allconc2}
\vspace{0.05in}\centering
{
\begin{tabular}{c|c|c}
\hline
Layer & kernel, stride, padding & activation\\
\hline
conv1 & 11$\times$11$\times$3$\times$96, 4,0 & $\cat$/$\abs$\\
\hline
conv2 & 1$\times$1$\times$192/96$\times$96, 1,0 & $\cat$/$\abs$\\
\hline
conv3 & 3$\times$3$\times$192/96$\times$96, 2,0 & $\cat$/$\abs$\\
\hline
conv4 & 5$\times$5$\times$192/96$\times$256, 1, 2 & $\cat$/$\abs$\\
\hline
conv5 & 1$\times$1$\times$512/256$\times$256, 1,0 & $\cat$/$\abs$\\
\hline
conv6 & 3$\times$3$\times$512/256$\times$256, 2,0 & $\cat$/$\abs$\\
\hline
conv7 & 3$\times$3$\times$512/256$\times$384, 1 ,1& $\cat$/$\abs$\\
\hline
conv8 & 1$\times$1$\times$768/384$\times$384, 1,0 & $\relu$\\
\hline
conv9 & 3$\times$3$\times$384$\times$384, 2,1 & $\relu$\\
\hline
    & dropout with ratio $0.25$ & \\
\hline
conv10 & 3$\times$3$\times$384$\times$1024, 1,1 & $\relu$\\
\hline
conv11 & 1$\times$1$\times$1024$\times$1024, 1,0 & $\relu$\\
\hline
conv12 & 1$\times$1$\times$1024$\times$1000, 1 & $\relu$\\
\hline
pool & 6$\times$6 average-pooling & \\
\hline
\hline
\end{tabular}
}
\end{table}

\begin{table}[t]
\caption{$\cat$/$\abs$ (conv1-9) for ImageNet}\label{allconc3}
\vspace{0.05in}\centering
{
\begin{tabular}{c|c|c}
\hline
Layer & kernel, stride, padding & activation\\
\hline
conv1 & 11$\times$11$\times$3$\times$96, 4,0 & $\cat$/$\abs$\\
\hline
conv2 & 1$\times$1$\times$192/96$\times$96, 1,0 & $\cat$/$\abs$\\
\hline
conv3 & 3$\times$3$\times$192/96$\times$96, 2,0 & $\cat$/$\abs$\\
\hline
conv4 & 5$\times$5$\times$192/96$\times$256, 1, 2 & $\cat$/$\abs$\\
\hline
conv5 & 1$\times$1$\times$512/256$\times$256, 1,0 & $\cat$/$\abs$\\
\hline
conv6 & 3$\times$3$\times$512/256$\times$256, 2,0 & $\cat$/$\abs$\\
\hline
conv7 & 3$\times$3$\times$512/256$\times$384, 1 ,1& $\cat$/$\abs$\\
\hline
conv8 & 1$\times$1$\times$768/384$\times$384, 1,0 & $\cat$/$\abs$\\
\hline
conv9 & 3$\times$3$\times$768/384$\times$384, 2,1 & $\cat$/$\abs$\\
\hline
    & dropout with ratio $0.25$ & \\
\hline
conv10 & 3$\times$3$\times$768/384$\times$1024, 1,1 & $\relu$\\
\hline
conv11 & 1$\times$1$\times$1024$\times$1024, 1,0 & $\relu$\\
\hline
conv12 & 1$\times$1$\times$1024$\times$1000, 1 & $\relu$\\
\hline
pool & 6$\times$6 average-pooling & \\
\hline
\hline
\end{tabular}
}\vspace{0.2in}
\caption{$\cat$ (all) for ImageNet}\label{allconch}
\vspace{0.05in}\centering
{
\begin{tabular}{c|c|c}
\hline
Layer & kernel, stride, padding & activation\\
\hline
conv1 & 11$\times$11$\times$3$\times$48, 4,0 & $\cat$\\
\hline
conv2 & 1$\times$1$\times$96$\times$48, 1,0 & $\cat$\\
\hline
conv3 & 3$\times$3$\times$96$\times$48, 2,0 & $\cat$\\
\hline
conv4 & 5$\times$5$\times$96$\times$128, 1, 2 & $\cat$\\
\hline
conv5 & 1$\times$1$\times$256$\times$128, 1,0 & $\cat$\\
\hline
conv6 & 3$\times$3$\times$256$\times$128, 2,0 & $\cat$\\
\hline
conv7 & 3$\times$3$\times$256$\times$192, 1 ,1& $\cat$\\
\hline
conv8 & 1$\times$1$\times$384$\times$192, 1,0 & $\cat$\\
\hline
conv9 & 3$\times$3$\times$384$\times$192, 2,1 & $\cat$\\
\hline
    & dropout with ratio $0.25$ & \\
\hline
conv10 & 3$\times$3$\times$384$\times$512, 1,1 & $\cat$\\
\hline
conv11 & 1$\times$1$\times$512$\times$512, 1,0 & $\cat$\\
\hline
conv12 & 1$\times$1$\times$512$\times$1000, 1 & $\cat$\\
\hline
pool & 6$\times$6 average-pooling & \\
\hline
\hline
\end{tabular}
}\vspace{0.2in}
\caption{$\cat$ (conv1,4,7) for ImageNet}\label{allconchonly}
\vspace{0.05in}\centering
{
\begin{tabular}{c|c|c}
\hline
Layer & kernel, stride, padding & activation\\
\hline
conv1 & 11$\times$11$\times$3$\times$48, 4,0 & $\cat$\\
\hline
conv2 & 1$\times$1$\times$96$\times$96, 1,0 & $\relu$\\
\hline
conv3 & 3$\times$3$\times$96$\times$96, 2,0 & $\relu$\\
\hline
conv4 & 5$\times$5$\times$96$\times$128, 1, 2 & $\cat$\\
\hline
conv5 & 1$\times$1$\times$256$\times$256, 1,0 & $\relu$\\
\hline
conv6 & 3$\times$3$\times$256$\times$256, 2,0 & $\relu$\\
\hline
conv7 & 3$\times$3$\times$256$\times$192, 1 ,1& $\cat$\\
\hline
conv8 & 1$\times$1$\times$384$\times$384, 1,0 & $\relu$\\
\hline
conv9 & 3$\times$3$\times$384$\times$384, 2,1 & $\relu$\\
\hline
    & dropout with ratio $0.25$ & \\
\hline
conv10 & 3$\times$3$\times$384$\times$1024, 1,1 & $\relu$\\
\hline
conv11 & 1$\times$1$\times$1024$\times$1024, 1,0 & $\relu$\\
\hline
conv12 & 1$\times$1$\times$1024$\times$1000, 1 & $\relu$\\
\hline
pool & 6$\times$6 average-pooling & \\
\hline
\hline
\end{tabular}
}
\end{table}

\onecolumn
\section{Image Reconstruction}
\label{image_recon}
In this section, we provide more image reconstruction examples.
%
\begin{figure*}[htbp]
\centering
\subfigure[Original image]{\includegraphics[width=0.19\textwidth]{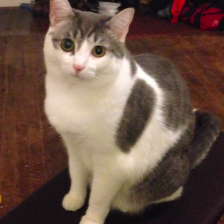}}\hspace{0.02in}
\subfigure[conv1]{\includegraphics[width=0.19\textwidth]{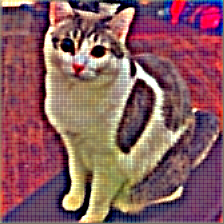}}\hspace{0.02in}
\subfigure[conv2]{\includegraphics[width=0.19\textwidth]{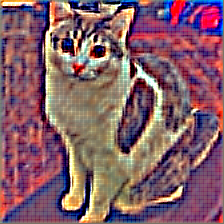}}\hspace{0.02in}
\subfigure[conv3]{\includegraphics[width=0.19\textwidth]{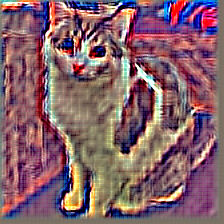}}\hspace{0.02in}
\subfigure[conv4]{\includegraphics[width=0.19\textwidth]{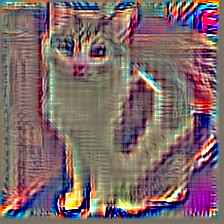}}
\\
\subfigure[Original image]{\includegraphics[width=0.19\textwidth]{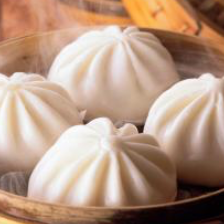}}\hspace{0.02in}
\subfigure[conv1]{\includegraphics[width=0.19\textwidth]{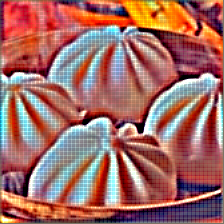}}\hspace{0.02in}
\subfigure[conv2]{\includegraphics[width=0.19\textwidth]{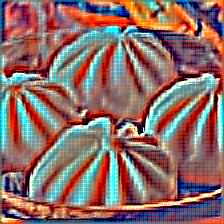}}\hspace{0.02in}
\subfigure[conv3]{\includegraphics[width=0.19\textwidth]{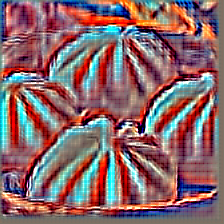}}\hspace{0.02in}
\subfigure[conv4]{\includegraphics[width=0.19\textwidth]{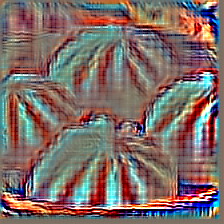}}
\\
\subfigure[Original image]{\includegraphics[width=0.19\textwidth]{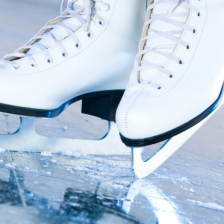}}\hspace{0.02in}
\subfigure[conv1]{\includegraphics[width=0.19\textwidth]{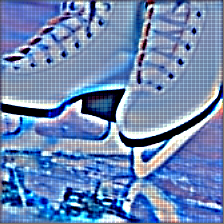}}\hspace{0.02in}
\subfigure[conv2]{\includegraphics[width=0.19\textwidth]{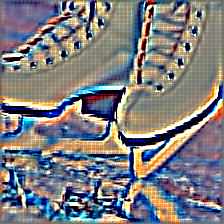}}\hspace{0.02in}
\subfigure[conv3]{\includegraphics[width=0.19\textwidth]{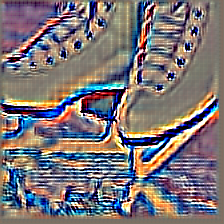}}\hspace{0.02in}
\subfigure[conv4]{\includegraphics[width=0.19\textwidth]{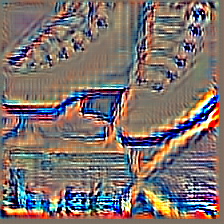}}
\\
\subfigure[Original image]{\includegraphics[width=0.19\textwidth]{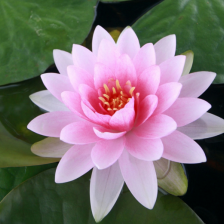}}\hspace{0.02in}
\subfigure[conv1]{\includegraphics[width=0.19\textwidth]{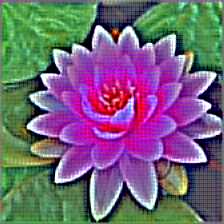}}\hspace{0.02in}
\subfigure[conv2]{\includegraphics[width=0.19\textwidth]{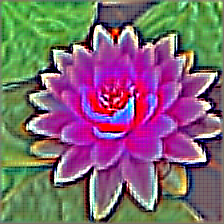}}\hspace{0.02in}
\subfigure[conv3]{\includegraphics[width=0.19\textwidth]{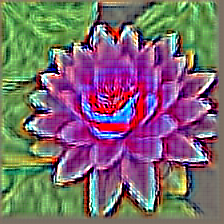}}\hspace{0.02in}
\subfigure[conv4]{\includegraphics[width=0.19\textwidth]{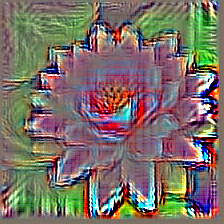}}
\end{figure*}

\end{document}